\documentclass{article}

\usepackage[nonatbib, preprint]{neurips_2021}
\usepackage[numbers]{natbib}

\usepackage[utf8]{inputenc} 
\usepackage[T1]{fontenc}    
\usepackage{hyperref}       
\usepackage{url}            
\usepackage{booktabs}       
\usepackage{amsfonts}       
\usepackage{nicefrac}       
\usepackage{microtype}      
\usepackage{xcolor}         

\usepackage{graphicx}
\usepackage{subcaption}
\usepackage{wrapfig}
\usepackage{caption}
\usepackage{lipsum}
\usepackage{xcolor}
\usepackage{amsmath}
\usepackage{amssymb}
\usepackage{chngcntr}
\usepackage{pifont}
\usepackage{amsthm}
\makeatletter
\def\thm@space@setup{%
  \thm@preskip=0.5cm
  \thm@postskip=0.4cm 
}
\makeatother
\newtheorem{definition}{Definition}
\newtheorem{proposition}{Proposition}
\counterwithin*{proposition}{section}
\newtheorem{lemma}{Lemma}
\newcommand{\obj}{\mathbb{J}} 

\newcommand{\varz}{Z} 
\newcommand{\valz}{\mathbf{z}} 
\newcommand{\setz}{\mathcal{Z}} 
\newcommand{\vars}{S}
\newcommand{\vals}{\mathbf{s}}
\newcommand{\sets}{\mathcal{S}}
\newcommand{\vara}{A}
\newcommand{\vala}{\mathbf{a}}
\newcommand{\seta}{\mathcal{A}}
\newcommand{\varr}{R} 
\newcommand{\valr}{r} 
\newcommand{\funcr}{r} 
\newcommand{\setr}{\mathcal{R}} 
\newcommand{\rep}{\phi_{\setz}} 
\newcommand{\setrep}{\Phi} 
\newcommand{\return}{\bar{R}}

\newcommand{\fwd}{\mathbb{J}_{fwd}}
\newcommand{\inv}{\mathbb{J}_{inv}}
\newcommand{\cpc}{\mathbb{J}_{state}}

\usepackage{amsmath,amsfonts,bm}

\def\eqref#1{equation~\ref{#1}}

\def\1{\bm{1}}

\DeclareMathAlphabet{\mathsfit}{\encodingdefault}{\sfdefault}{m}{sl}
\SetMathAlphabet{\mathsfit}{bold}{\encodingdefault}{\sfdefault}{bx}{n}

\newcommand{\E}{\mathbb{E}}

\DeclareMathOperator*{\argmax}{arg\,max}

\title{Which Mutual-Information Representation Learning Objectives are Sufficient for Control?}

\author{
  Kate Rakelly\thanks{Correspondence to rakelly@eecs.berkeley.edu} \qquad Abhishek Gupta \qquad Carlos Florensa \qquad Sergey Levine \\
  University of California, Berkeley \\
  \{rakelly, abhigupta, florensacc, svlevine\}@eecs.berkeley.edu
}

\begin{document}

\maketitle

\begin{abstract}
Mutual information maximization provides an appealing formalism for learning representations of data. In the context of reinforcement learning (RL), such representations can accelerate learning by discarding irrelevant and redundant information, while retaining the information necessary for control. Much of the prior work on these methods has addressed the practical difficulties of estimating mutual information from samples of high-dimensional observations, while comparatively less is understood about \emph{which} mutual information objectives yield representations that are sufficient for RL from a theoretical perspective. In this paper, we formalize the sufficiency of a state representation for learning and representing the optimal policy, and study several popular mutual-information based objectives through this lens. Surprisingly, we find that two of these objectives can yield insufficient representations given mild and common assumptions on the structure of the MDP. We corroborate our theoretical results with empirical experiments on a simulated game environment with visual observations. 
\end{abstract}

\section{Introduction}
Deep reinforcement learning (RL) algorithms are in principle capable of learning policies from high-dimensional observations, such as camera images~\citep{mnih2013playing, lee2019slac, kalashnikov2018scalable}. However, policy learning in practice faces a bottleneck in acquiring useful representations of the observation space~\citep{shelhamer2016loss}. 
State representation learning approaches aim to remedy this issue by learning structured and compact representations on which to perform RL. 
A useful state representation should be \emph{sufficient} to learn and represent the optimal policy or the optimal value function, while discarding irrelevant and redundant information. 
Understanding whether or not an objective is guaranteed to yield sufficient representations is important, because insufficient representations make it impossible to solve certain problems. 
For example, an autonomous vehicle would not be able to navigate safely if its state representation did not contain information about the color of the stoplight in front of it.
With the increasing interest in leveraging offline datasets to learn representations for RL~\citep{finn2016deep, kipf2019contrastive, stooke2020decoupling}, the question of sufficiency becomes even more important to understand if the representation is capable of representing policies and value functions for downstream tasks. 

While a wide range of representation learning objectives have been proposed in the literature~\citep{lesort2018state}, in this paper we focus on analyzing sufficiency of representations learned by maximizing the mutual information (MI) between random variables.
Prior work has proposed many different MI objectives involving the variables of states, actions, and rewards at different time-steps~\citep{anand2019unsupervised, oord2018representation, pathak2017curiosity, shelhamer2016loss}.
While much prior work has focused on how to optimize these various mutual information objectives in high dimensions~\citep{song2019understanding, belghazi2018mine, oord2018representation, hjelm2018learning}, we focus instead on their ability to yield theoretically sufficient representations.
We find that two commonly used objectives are insufficient for the general class of MDPs, in the most general case, 
and prove that another typical objective is sufficient. 
We illustrate the analysis with both didactic examples in which MI can be computed exactly and deep RL experiments in which we approximately maximize the MI objective to learn representations of visual inputs.
The experimental results corroborate our theoretical findings, and demonstrate that the sufficiency of a representation can have a substantial impact on the performance of an RL agent that uses that representation.
This paper provides guidance to the deep RL practitioner on when and why objectives may be expected to work well or fail, and also provides a formal framework to analyze newly proposed representation learning objectives based on MI.

\section{Related Work}

\label{sec:rw}
In this paper, we analyze several widely used mutual information objectives for control.
In this section we first review MI-based unsupervised learning, then the application of these techniques to the RL setting.
Finally, we discuss alternative perspectives on representation learning in RL.

\textbf{Mutual information-based unsupervised learning.}
A common technique for unsupervised representation learning based on the InfoMax principle~\citep{linsker1988self, bell1995information} is to maximize the MI between the input and its latent representation subject to domain-specific constraints~\citep{becker1992self}.
This technique has been applied to learn representations for natural language~\citep{devlin2019bert}, video~\citep{sun2019contrastive}, and images~\citep{bachman2019learning, hjelm2018learning} and even policy learning via RL in high dimensions~\cite{srinivas2020curl}.
To address the difficulties of estimating MI from samples~\citep{mcallester2018formal} and with high-dimensional inputs~\citep{song2019understanding},
much recent work has focused on improving MI estimation via variational methods~\citep{song2019understanding, poole2019variational, oord2018representation, belghazi2018mine}.
In this work we are concerned with analyzing the MI objectives, and not the estimation method.
In our experiments with image observations, we use noise contrastive estimation methods~\citep{gutmann2010noise}, though other choices could also suffice.

\textbf{Mutual information objectives in RL.}
Reinforcement learning adds aspects of temporal structure and control to the standard unsupervised learning problem discussed above (see Figure~\ref{fig:pgm}).
This structure can be leveraged by maximizing MI between sequential states, actions, or combinations thereof.
Some works omit the action, maximizing the MI between current and future states~\citep{anand2019unsupervised, oord2018representation, stooke2020decoupling}.
Several prior works~\citep{nachum2018near, schwarzerdata, shu2020predictive, mazoure2020deep} maximize MI objectives that closely resemble the forward information objective we introduce in Section~\ref{sec:obj}, whiel others optimize related objectives by learning latent forward dynamics models~\citep{watter2015embed, karl2016deep, zhang2018solar, hafner2019learning, lee2019slac}. 
Multi-step inverse models, closely related to the inverse information objective (Section~\ref{sec:obj}),
have been used to learn control-centric representations~\citep{yuunsupervised, gregor2016variational}.
Single-step inverse models have been deployed as regularization of forward models~\citep{zhang2018decoupling, agrawal2016learning} and as an auxiliary loss for policy gradient RL~\citep{shelhamer2016loss, pathak2017curiosity}.
The MI objectives that we study have also been used as reward bonuses to improve exploration, without impacting the representation, in the form of empowerment~\citep{klyubin2008keep, klyubin2005empowerment, mohamed2015variational, leibfried2019unified} and information-theoretic curiosity~\citep{still2012information}.

\textbf{Representation learning for reinforcement learning.}
In RL, the problem of finding a compact state space has been studied as state aggregation or abstraction~\citep{bean1987aggregation, li2006towards}.
Abstraction schemes include bisimulation~\citep{givan2003equivalence}, homomorphism~\citep{ravindran2003smdp}, utile distinction~\citep{mccallum1996reinforcement}, and policy irrelevance~\citep{jong2005state}. 
While efficient algorithms exist for MDPs with known transition models for some abstraction schemes such as bisimulation~\citep{ferns2006methods, givan2003equivalence}, in general obtaining error-free abstractions is impractical for most problems of interest.
For approximate abstractions prior work has bounded the sub-optimality of the policy~\citep{bertsekas1988adaptive, dean1997model, abel2016near} as well as the sample efficiency~\citep{lattimore2019learning, van2019comments, du2019good}, with some results in the deep learning setting~\citep{gelada2019deepmdp, nachum2018near}.
In this paper, we focus on whether a representation can be used to learn the optimal policy, and not the tractability of learning.
~\citet{li2006towards} shares this focus; while they establish convergence properties of $Q$-learning with representations satisfying different notions of sufficiency, we leverage their $Q^*$-sufficiency criteria to evaluate representations learned via MI-based objectives. 
Alternative approaches to representation learning for RL include priors based on the structure of the physical world~\citep{jonschkowski2015learning} or heuristics such as disentanglement~\citep{thomas2017independently}, meta-learning general value functions~\citep{veeriah2019discovery}, predicting multiple value functions~\citep{bellemare2019geometric, fedus2019hyperbolic, jaderberg2016reinforcement} and predicting domain-specific measurements~\citep{mirowski2019learning, dosovitskiy2016learning}.
We restrict our analysis to objectives that can be expressed as MI-maximization.
In our paper we focus on the representation learning problem, disentangled from exploration, a strategy shared by prior works~\cite{finn2016deep, kipf2019contrastive, stooke2020decoupling, zhan2020framework}.

\section{Representation Learning for RL}
\label{sec:def}
The goal of representation learning for RL is to learn a compact representation of the state space that discards irrelevant and redundant information, while still retaining sufficient information to represent policies and value functions needed for learning. In this section we formalize this problem, and propose and define the concept of sufficiency to evaluate the usefulness of a representation.

\subsection{Preliminaries}
\label{sec:prelims}
We begin with brief preliminaries of reinforcement learning and mutual information. 

\paragraph{Reinforcement learning.} A Markov decision process (MDP) is defined by the tuple $(\sets, \seta, \mathcal{T}, \valr)$, where $\sets$ is the set of states, $\seta$ the set of actions, $\mathcal{T}: \sets \times \seta \times \sets \rightarrow [0, 1]$ the state transition distribution, and $\valr: \sets \rightarrow \mathbb{R}$ the reward function~\footnote{We restrict our attention to MDPs where the reward can be expressed as a function of the state, which is fairly standard across a broad set of real world RL problems}.
We will use capital letters to refer to random variables and lower case letters to refer to values of those variables (e.g., $\vars$ is the random variable for the state and $\vals$ is a specific state).
Throughout our analysis we will often be interested in multiple reward functions, and denote a set of reward functions as $\setr$.
The objective of RL is to find a policy that maximizes the sum of discounted returns $\bar{R}$ for a given reward function $r$, and we denote this optimal policy as $\pi^*_{\funcr} = \argmax_{\pi} \E_{\pi}[\sum_t \gamma^t r(\vars_t, \vara_t)]$ for discount factor $\gamma$.
We also define the optimal $Q$-function as $Q^*_{\funcr}(\vals_t, \vala_t) = \E_{\pi^*} [\sum_{t=1}^{\infty} \gamma^t r(\vars_t, \vara_t) | \vals_t, \vala_t]$.
The optimal $Q$-function satisfies the recursive Bellman equation, $Q^*_{\funcr}(\vals_t, \vala_t) = r(\vals_t, \vala_t) + \gamma \E_{p(\vals_{t+1} | \vals_t, \vala_t)} \argmax_{\vala_{t+1}} Q^*_{\funcr}(\vals_{t+1}, \vala_{t+1})$. 
An optimal policy and the optimal Q-function are related according to $\pi^*(\vals) = \argmax_{\vala} Q^*(\vals, \vala)$. 

\paragraph{Mutual information.} In information theory, the mutual information (MI) between two random variables, $X$ and $Y$, is defined as \citep{cover1999elements}:
\begin{equation}
    I(X; Y) = \E_{p(x, y)} \log \frac{p(x, y)}{p(x)p(y)} = H(X) - H(X | Y).
\end{equation}
The first definition indicates that MI can be understood as a relative entropy (or KL-divergence), while the second underscores the intuitive notion that MI measures the reduction in the uncertainty of one random variable from observing the value of the other.

\paragraph{Representation learning for RL.} 
\begin{wrapfigure}{R}{0.3\textwidth}
    \vspace{-0.4cm}
    \includegraphics[height=0.115\textheight]{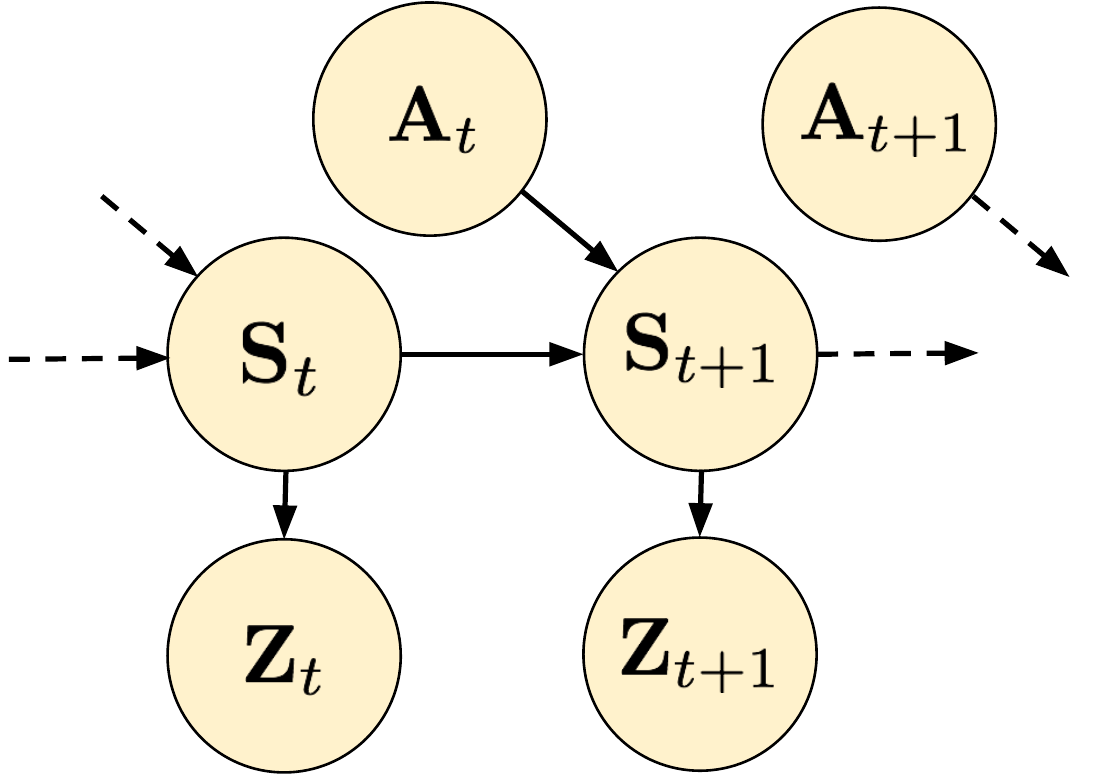}
  \caption{
  Probabilistic graphical model illustrating
  the state representation learning problem: estimate representation $\varz$ from original state $\vars$. 
  }
  \label{fig:pgm}
    \vspace{-0.3cm}
\end{wrapfigure}
The goal of representation learning for RL is to find a compact representation of the state space that discards
details in the state that are not relevant for representing the policy or value function, while preserving task-relevant information (see Figure~\ref{fig:pgm}).
While state aggregation methods typically define deterministic rules to group states in the representation~\citep{bean1987aggregation, li2006towards}, MI-based representation learning methods used for deep RL treat the representation as a random variable~\citep{nachum2018near, oord2018representation, pathak2017curiosity}.
Accordingly, we formalize a representation as a stochastic mapping between original state space and representation space. 
\vspace{-.1in}
\begin{definition}
A \textbf{stochastic representation} $\rep(\vals)$ is a mapping from states $\vals\in\sets$ to a probability distribution $p(\varz|\vars=\vals)$ over elements of a new representation space $z\in\setz$.
\label{def:rep}
\end{definition}
\vspace{-.1in}
In this work we consider learning state representations from data by maximizing an objective $\obj$.
Given an objective $\obj$, we define the set of representations that maximize this objective as $\setrep_{\obj} = \{\rep ~|~\rep \in \argmax \obj (\phi)$. 
Unlike problem formulations for partially observed settings~\citep{watter2015embed, hafner2019learning, lee2019slac}, we assume that $\vars$ is a Markovian state; therefore the representation for a given state is conditionally independent of the past states, a common assumption in the state aggregation literature~\citep{bean1987aggregation, li2006towards}.
See Figure~\ref{fig:pgm} for a depiction of the graphical model.

\subsection{Sufficient Representations for Reinforcement Learning}

We now turn to the problem of evaluating stochastic representations for RL. Intuitively, we expect a useful state representation to be capable of representing an optimal policy in the original state space. 
\vspace{-.1in}
\begin{definition}
A representation $\rep$ is \textbf{$\pi^*$-sufficient} with respect to a set of reward functions $\setr$ if $\forall \valr \in \setr$, $\rep(\vals_1) = \rep(\vals_2) \implies \pi^*_{r}(\vara | \vals_1) = \pi^*_{r}(\vara | \vals_2)$.
\end{definition}
\label{def:suff-pi}
\vspace{-.1in}
When a stochastic representation $\rep$ produces the same distribution over the representation space for two different states $\vals_1$ and $\vals_2$ we say it \textit{aliases} these states. Unfortunately, as already proven in Theorem 4 of~\citet{li2006towards} for the more restrictive case of deterministic representations, being able to represent the optimal policy does not guarantee that it can be learned via RL in the representation space. 
Accordingly, we define a stricter notion of sufficiency that \emph{does} guarantee the convergence of Q-learning to the optimal policy in the original state space (refer to Theorem 4 of ~\citet{li2006towards} for the proof of this). 
\vspace{-.1in}
\begin{definition}
A representation $\rep$ is \textbf{$Q^*$-sufficient} with respect to a set of reward functions $\setr$ if $\forall \valr \in \setr$, $\rep(\vals_1) = \rep(\vals_2) \implies \forall \vala \text{,} Q^*_{r}(\vala , \vals_1) = Q^*_{r}(\vala , \vals_2)$.
\end{definition}
\label{def:suff-q}
\vspace{-.1in}
Note that $Q^*$-sufficiency implies $\pi^*$-sufficiency since an optimal policy can be recovered from the optimal Q-function via $\pi^*_{r}(s) = \argmax_a Q^*_{r}(s, a)$~\citep{sutton2018reinforcement}; however the converse is not true. 
We emphasize that while $Q^*$-sufficiency guarantees convergence, it does not guarantee tractability, which has been explored in prior work~\citep{lattimore2019learning, du2019good}.

We will further say that an \emph{objective} $\obj$ is sufficient with respect to some set of reward functions $\setr$ if all the representations that maximize that objective $\setrep_{\obj}$ are sufficient with respect to every element of $\setr$ according to the definition above. 
Surprisingly, we will demonstrate that not all commonly used objectives satisfy this basic qualification even when $\setr$ contains a single known reward function.

\section{Mutual Information for Representation Learning in RL}
\label{sec:obj}

In our study, we consider several MI objectives proposed in the literature.

\paragraph{Forward information:} 
A commonly sought characteristic of a state representation is to ensure it retains maximum predictive power over future state representations. 
This property is satisfied by representations maximizing the following MI objective,
\begin{equation}
    \fwd = I(\varz_{t+1}; \varz_t, \vara_t) = H(\varz_{t+1}) - H(\varz_{t+1} | \varz_t, \vara_t). \label{eq:mi_fwd}
\end{equation}
We suggestively name this objective ``forward information'' due to the second term, which is the entropy of the forward dynamics distribution.
This objective and closely related ones have been used in prior works~\citep{nachum2018near, schwarzerdata, shu2020predictive, mazoure2020deep}.

\paragraph{State-only transition information:}
Several popular methods~\citep{oord2018representation, anand2019unsupervised, stooke2020decoupling} optimize a similar objective, but do not include the action:
\begin{equation}
    \cpc = I(\varz_{t+k}; \varz_t) = H(\varz_{t+k}) - H(\varz_{t+k} | \varz_t). \label{eq:mi_state}
\end{equation}
As we will show, the exclusion of the action can have a profound effect on the characteristics of the resulting representations.

\paragraph{Inverse information:} 
Another commonly sought characteristic of state representations is to retain maximum predictive power of the action distribution that could have generated an observed transition from $\vals_t$ to $\vals_{t+1}$.
Such representations can be learned by maximizing the following information theoretic objective:
\begin{equation}
    \inv = I(\vara_t; \varz_{t+k} | \varz_t) = H(\vara_t | \varz_t) - H(\vara_t | \varz_t, \varz_{t+k}) 
    \label{eq:mi_inv}
\end{equation}
We suggestively name this objective ``inverse information'' due to the second term, which is the entropy of the inverse dynamics. 
A wide range of prior work learns representations by optimizing closely related objectives~\citep{gregor2016variational, shelhamer2016loss, agrawal2016learning, pathak2017curiosity, yuunsupervised, zhang2018decoupling}.
Intuitively, inverse models allow the representation to capture only the elements of the state that are necessary to predict the action, allowing the discard of potentially irrelevant information.

\section{Sufficiency Analysis}
\label{sec:sufficiency}

In this section we analyze the sufficiency for control of representations obtained by maximizing each objective presented in Section~\ref{sec:obj}. 
To focus on the representation learning problem, we decouple it from RL by assuming access to a dataset of transitions collected with a policy that reaches all states with non-zero probability, which can then be used to learn the desired representation. 
We also assume that distributions, such as the dynamics or inverse dynamics, can be modeled with arbitrary accuracy, and that the maximizing set of representations for a given objective can be computed. 
While these assumptions might be relaxed in any practical RL algorithm, and exploration plays a confounding role, the ideal assumptions underlying our analysis provide the best-case scenario for objectives to yield provably sufficient representations.  
In other words, objectives found to be provably insufficient under ideal conditions will continue to be insufficient under more realistic assumptions.

\subsection{Forward Information}
\label{sec:fwd}

In this section we show that a representation that maximizes $\fwd$ is sufficient for optimal control under any reward function. 
This result aligns with the intuition that a representation that captures forward dynamics can represent everything predictable in the state space, and can thus be used to learn the optimal policy for any task. 
This strength can also be a weakness if there are many predictable elements that are irrelevant for downstream tasks, since the representation retains more information than is needed for the task. 
Note that the representation can still discard information in the original state, such as independent random noise at each timestep.

\begin{proposition}
$\fwd$ is sufficient for all reward functions.
\label{prop:fwd_suff}
\end{proposition}

\begin{proof}
\emph{(Sketch)} 
We first show in Lemma \ref{lemma:1} that if $\varz_t, \vara_t$ are maximally informative of $\varz_{t+k}$, they are also maximally informative of the return $\bar{R}_t$.
Thanks to the Markov structure, we then show in Lemma \ref{lemma:2} that $\E_{p(\varz_t | \vars_t = \vals)} p(\return_t | \varz_t, \vara_t) = p(\return_t | \vars_t = \vals, \vara_t)$.
In other words, given $\rep$, additionally knowing $\vars$ doesn't change our belief about the future return. The $Q$-value is the expectation of the return, so $\varz$ has as much information about the $Q$-value as $\vars$. See Appendix~\ref{app:proof} for the proof.
\end{proof}

\subsection{State-Only Transition Information}

\begin{figure*}[b!]%
  \centering
  \includegraphics[width=5.5cm]{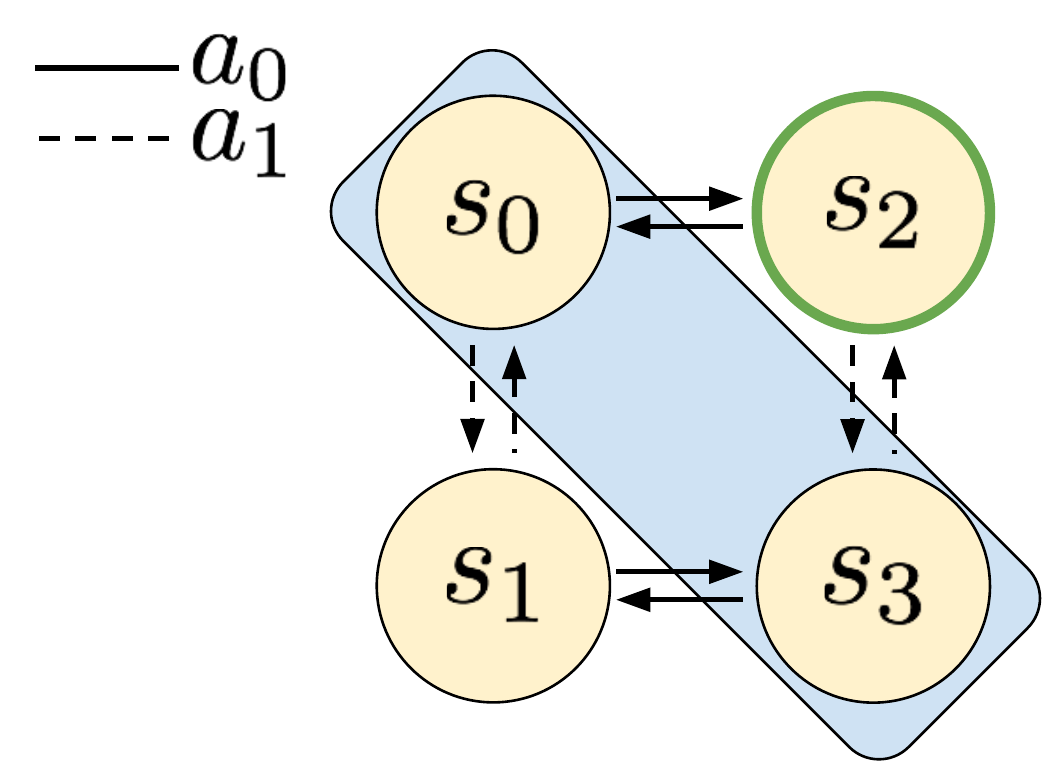}
  \includegraphics[width=7cm]{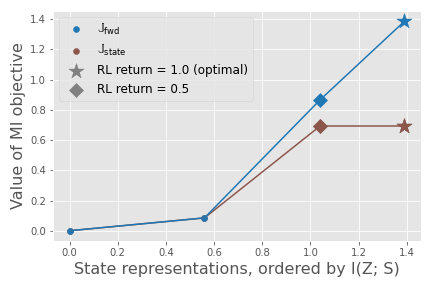}
  \caption{(left) A representation that aliases the states $\vals_0$ and $\vals_3$ into a single state maximizes $\cpc$ but is not sufficient to represent the optimal policy which must choose different actions in $\vals_0$ and $\vals_3$ to reach $\vals_2$ which yields reward. (right) Values of $\cpc$ and $\fwd$ for a few representative state representations, ordered by increasing $I(\varz; \vars)$. The representation that aliases $\vals_0$ and $\vals_3$ (plotted with a diamond) maximizes $\cpc$, but the policy learned with this representation may not be optimal (as shown here). The original state representation (plotted with a star) is sufficient.}%
  \label{fig:cpc_example}%
\end{figure*}
While $\cpc$ is closely related to $\fwd$, we now show that $\cpc$ is not sufficient.

\begin{proposition}
$\cpc$ is not sufficient for all reward functions. 
\end{proposition}
\vspace{-.2in}

\begin{proof}
We show this by counter-example with the deterministic-transition MDP defined in Figure~\ref{fig:cpc_example} (left).
For all four states, let the two actions $\vala_0$ and $\vala_1$ be equally likely under the policy distribution.
In this case, each state gives no information about which of the two possible next states is more likely; this depends on the action.
Therefore, a representation maximizing $\cpc$ is free to alias states with the same next-state distribution, such as $\vals_0$ and $\vals_3$.
An alternative view is that such a representation can maximize $\cpc = H(\varz_{t+k}) - H(\varz_{t+k} | \varz_t)$ by reducing both terms in equal amounts - aliasing $\vals_0$ and $\vals_3$ decreases the marginal entropy as well as the entropy of predicting the next state starting from $\vals_1$ or $\vals_2$.
However, this aliased representation is not capable of representing the optimal policy, which must distinguish $\vals_0$ and $\vals_3$ in order to choose the correct action to reach $\vals_2$, which yields reward.
\end{proof}

In Figure~\ref{fig:cpc_example} (right), we illustrate the insufficiency of $\cpc$ computationally, by computing the values of $\cpc$ and $\fwd$ for different state representations of the above MDP, ordered by increasing $I(\varz; \vars)$.
At the far left of the plot is the representation that aliases all states, while the original state representation is at the far right (plotted with a star).
The representation that aliases states $\vals_0$ and $\vals_3$ (plotted with a diamond) maximizes $\cpc$, but is insufficient to represent the optimal policy.
Value iteration run with this state representation achieves only half the optimal return ($0.5$ versus $1.0$).

\subsection{Inverse Information}
\label{sec:inv}
\begin{figure*}[b!]
  \centering
  \includegraphics[width=4.5cm]{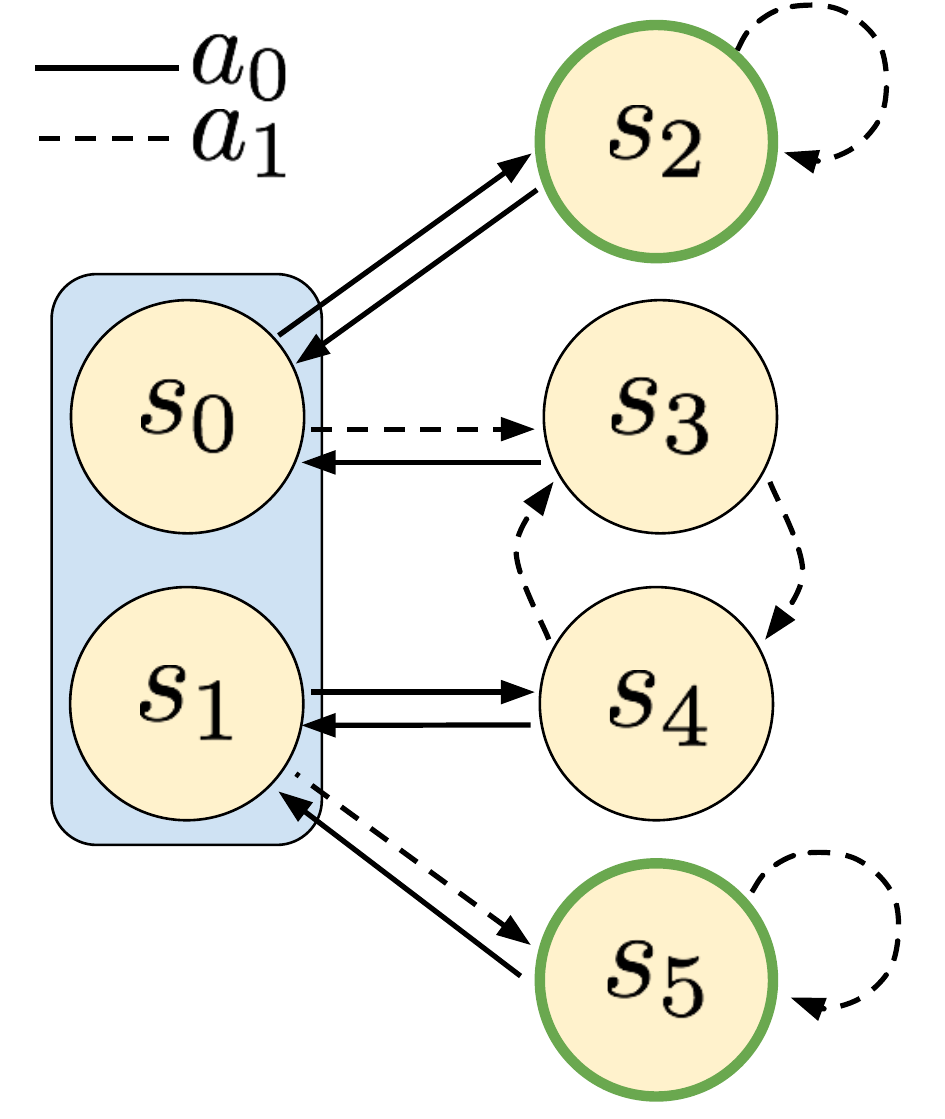}
  \includegraphics[width=7cm]{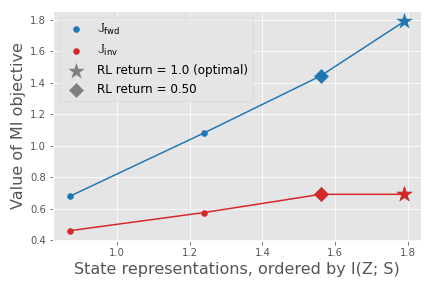}
  \caption{(left) In this MDP, a representation that aliases the states $\vals_0$ and $\vals_1$ into a single state maximizes $\inv$, yet is not sufficient to represent the optimal policy, which must distinguish between $\vals_0$ and $\vals_1$ in order to take a different action (towards the high-reward states outlined in green). (right) Values of $\inv$ and $\fwd$ for a few selected state representations, ordered by increasing $I(\varz; \vars)$. The representation that aliases $\vals_0$ and $\vals_1$ (plotted with a diamond) maximizes $\inv$, but is not sufficient to learn the optimal policy. Note that this counterexample holds also for $\inv + I(\varr; \varz)$.}%
  \label{fig:inv_example}%
\end{figure*}

Here we show that representations that maximize $\inv$ are not sufficient for control in all MDPs.
Intuitively, one way that the representation can be insufficient is by retaining only controllable state elements, while the reward function depends on state elements outside the agent's control. 
We then show that additionally representing the immediate reward is not enough to resolve this issue.

\begin{proposition}
$\inv$ is not sufficient for all reward functions. Additionally, adding $I(\varr_t; \varz_t)$ to the objective does not make it sufficient.
\end{proposition}

\begin{proof}
We show this by counter-example with the deterministic-transition MDP defined in Figure~\ref{fig:inv_example} (left). 
Now consider the representation that aliases the states $\vals_0$ and $\vals_1$. 
This state representation wouldn't be sufficient for control because the same actions taken from these two states lead to different next states, which have different rewards ($\vala_0$ leads to the reward from $\vals_0$ while $\vala_1$ leads to the reward from $\vals_1$).
However, this representation maximizes $\inv$ because, given each pair of states, the action is identifiable.
Interestingly, this problem cannot be remedied by simply requiring that the representation also be capable of predicting the reward at each state.
Indeed, the same insufficient representation from the above counterexample also maximizes this new objective as long as the reward at $\vals_0$ and $\vals_1$ are the same.
\end{proof}

Analogous to the preceding section, in Figure~\ref{fig:inv_example} (right), we plot the values of the objectives $\inv$ and $\fwd$ for state representations ordered by increasing $I(\varz; \vars)$.
The representation that aliases states $\vals_0$ and $\vals_1$ (plotted with a diamond) maximizes $\inv$, but is insufficient to represent the optimal policy; value iteration with this state representation achieves half the optimal return ($0.5$ versus $1.0$).

\section{Experiments}
\label{sec:exps}

To analyze whether the conclusions of our theoretical analysis hold in practice, we present experiments studying MI-based representation learning with image observations. 
We do not aim to show that any particular method is necessarily better or worse, but rather to determine whether the sufficiency arguments that we presented can translate into quantifiable performance differences in deep RL.

\subsection{Experimental Setup}
\label{sec:exp-setup}
To separate representation learning from RL, we first optimize each representation learning objective on a dataset of offline data, similar to the protocol in~\citet{stooke2020decoupling}.
Our datasets consist of $50$k transitions collected from a uniform random policy, which is sufficient to cover the state space in our environments.
We then freeze the weights of the state encoder learned in the first phase and train RL agents with the representation as state input. 
We perform our experiments on variations on the pygame~\citep{pygame} video game \textit{catcher}, in which the agent controls a paddle that it can move back and forth to catch fruit that falls from the top of the screen (see Figure~\ref{fig:envs}).
A positive reward is given when the fruit is caught and a negative reward when the fruit is not caught. 
The episode terminates after one piece of fruit falls.
We optimize $\fwd$ and $\cpc$ with noise contrastive estimation \citep{gutmann2010noise}, and $\inv$ by training an inverse model via maximum likelihood.
For the RL algorithm, we use the Soft Actor-Critic algorithm~\cite{haarnoja2018soft}, modified slightly for the discrete action distribution.
Please see Appendix~\ref{app:exp} for full experimental details.

\begin{wrapfigure}{R}{.5\textwidth}
\vspace{-5mm}
\centering
\includegraphics[width=3.2cm]{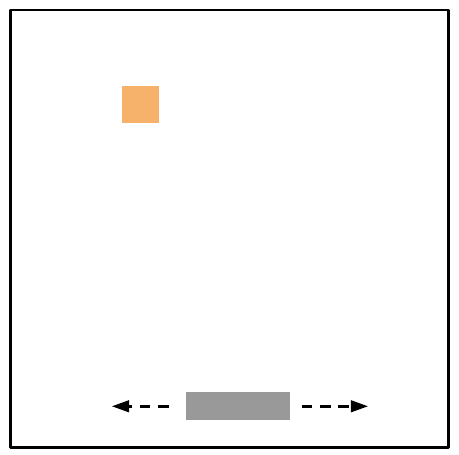}
\includegraphics[width=3.2cm]{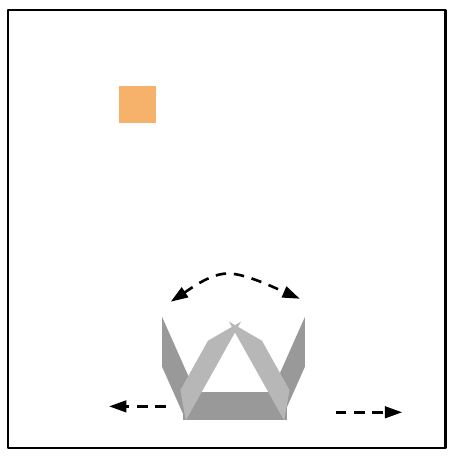}
\caption{(left) Original \textit{catcher} game in which the agent (grey paddle) moves left or right to catch fruit (yellow square) that falls from the top of the screen. (right) Variation \textit{catcher-grip} in which the agent is instantiated as a gripper, and must open the gripper to catch fruit.}
\vspace{-.1in}
\label{fig:envs}
\end{wrapfigure}

\subsection{Computational Results}
\label{sec:exp-results}
In principle, we expect that a representation learned with $\inv$ may not be sufficient to solve the \textit{catcher} game. 
Because the agent does not control the fruit, a representation maximizing $\inv$ might discard that information, thereby making it impossible to represent the optimal policy.
We observe in Figure~\ref{fig:results} (top left) that indeed the representation trained to maximize $\inv$ results in RL agents that converge slower and to a lower asymptotic expected return. 
Further, attempting to learn a decoder from the learned representation to the position of the falling fruit incurs a high error (Figure~\ref{fig:results}, bottom left), indicating that the fruit is not precisely captured by the representation.
The characteristics of this simulated game are representative of realistic tasks. 
Consider, for instance, an autonomous vehicle that is stopped at a stoplight.
Because the agent does not control the color of the stoplight, it may not be captured in the representation learned by $\inv$ and the resulting RL policy may choose to run the light.
\begin{figure}[!h]
    \centering
          \includegraphics[width=6.5cm, trim=0.25cm 0 0 0, clip]{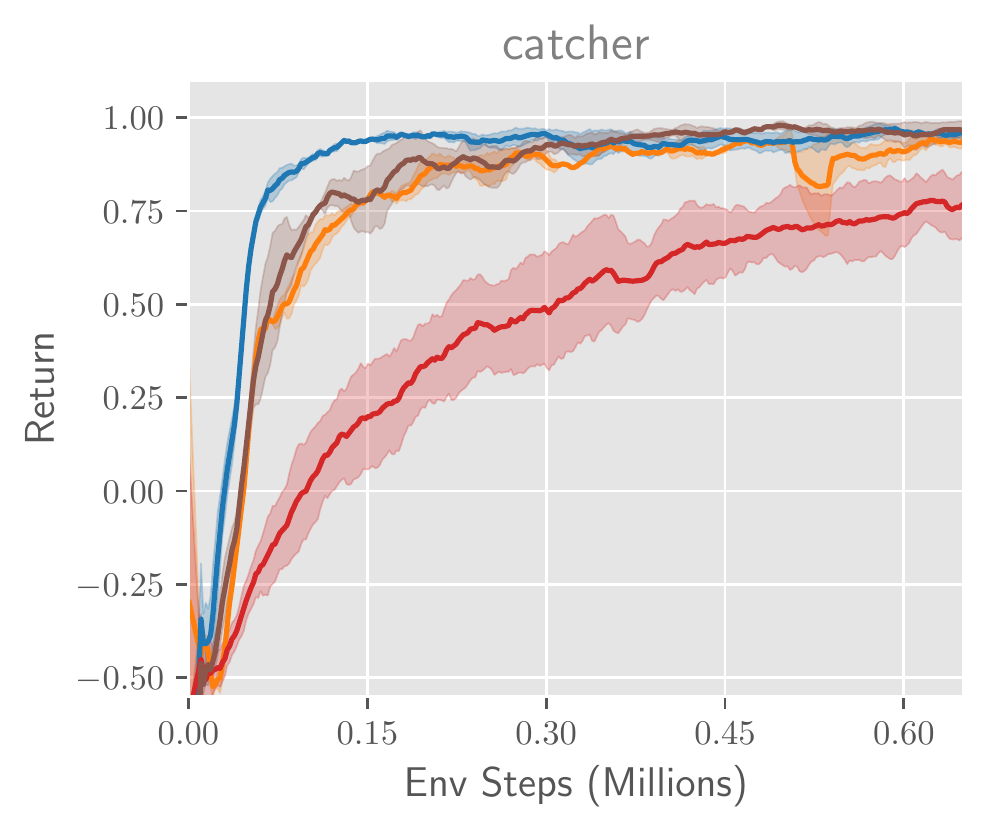} 
          \includegraphics[width=6.5cm, trim=0.25cm 0 0 0, clip]{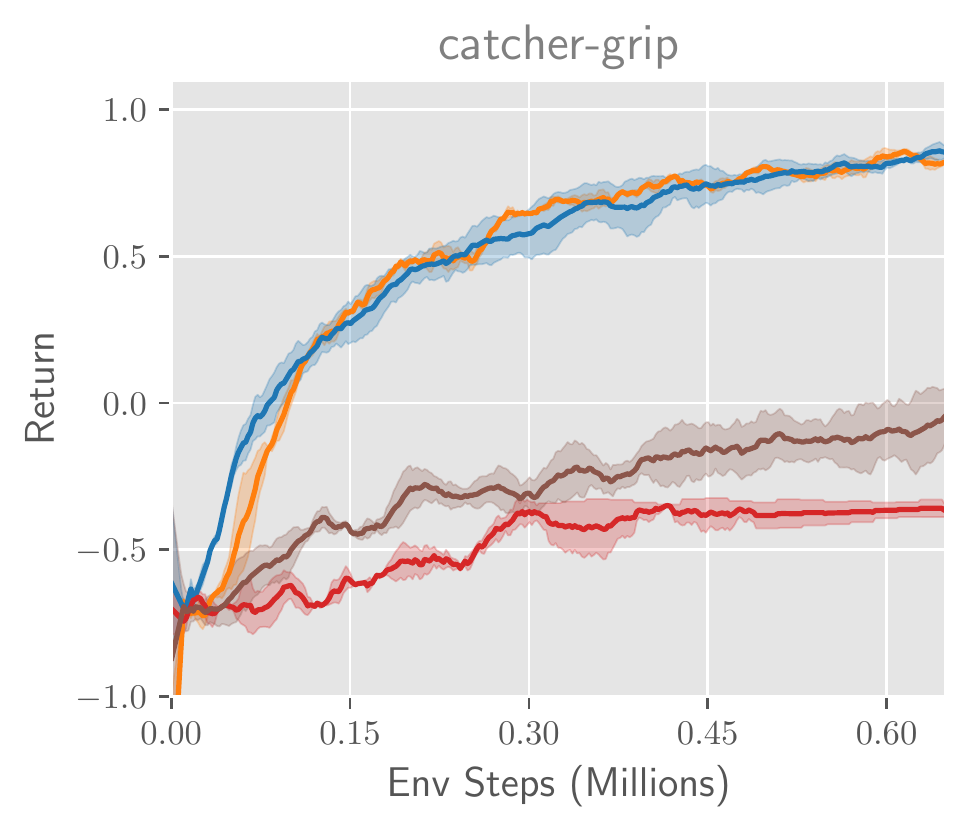}
          \includegraphics[height=.90cm]{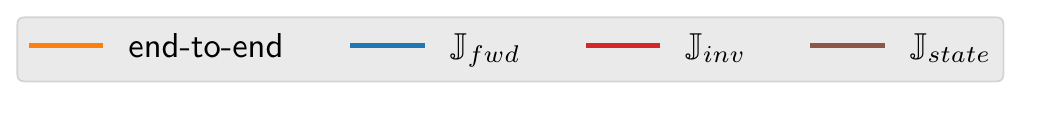} \\
          \hspace{.5cm}
          \begingroup
          \small
          \begin{tabular}{ ccc } \toprule
            Obj. & Agent Err. & Fruit Err.  \\
            \midrule
            $\fwd$ & 0.1 & 0.1 \\
            \midrule
            $\inv$ & 0.15 & 0.37 \\
            \bottomrule
            \end{tabular}
            \hspace{1cm}
          \begin{tabular}{ cccc }\toprule
            Obj. & Agent Err. & Fruit Err. & Grip Err.  \\
            \midrule
            $\fwd$ & 0.1 & 0.1 & 0.3 \\
            \midrule
            $\cpc$ & 0.1 & 0.1 & 0.46 \\
            \bottomrule
            \end{tabular}
        \endgroup
        \caption{ (left) Policy performance using learned representations as state inputs to RL, for the \textit{catcher} and \textit{catcher-grip} environments. (right) Error in predicting the positions of ground truth state elements from each learned representation. Representations maximizing $\inv$ need not represent the fruit, while representations maximizing $\cpc$ need not represent the gripper, leading these representations to perform poorly in \textit{catcher} and \textit{catcher-grip} respectively. }
\label{fig:results}

\end{figure}

In the second experiment, we consider a failure mode of $\cpc$.
We augment the paddle with a gripper that the agent controls and must be open in order to properly catch the fruit (see Figure~\ref{fig:envs}, right).
Since the change in the gripper is completely controlled by a single action, the current state contains no information about the state of the gripper in the future. 
Therefore, a representation maximizing $\cpc$ might alias states where the gripper is open with states where the gripper is closed.
Indeed, we see that the error in predicting the state of the gripper from the representation learned via $\cpc$ is about chance (Figure~\ref{fig:results}, bottom right).
This degrades the performance of an RL agent trained with this state representation since the best the agent can do is move under the fruit and randomly open or close the gripper (Figure~\ref{fig:results}, top right).
In the driving example, suppose turning on the headlights incurs positive reward if it's raining but negative reward if it's sunny.
The representation could fail to distinguish the state of the headlights, making it impossible to learn when to properly use them.
$\fwd$ produces useful representations in all cases, and is equally or more effective than learning representations purely from the RL objective alone (Figure~\ref{fig:results}).

\subsection{Increasing visual complexity via background distractors}
\label{exps:igr}
In this section, we test whether sufficiency of the state representation impacts agent performance in more visually complex environments by adding background distractors to the agent's observations.
We randomly generate images of $10$ circles of different colors and replace the background of the game with these images.
See Figure~\ref{fig:igr_obs} for examples of the agent's observations.

\begin{wrapfigure}{R}{.5\textwidth}
\vspace{-0.5cm}
\centering
\includegraphics[width=6cm]{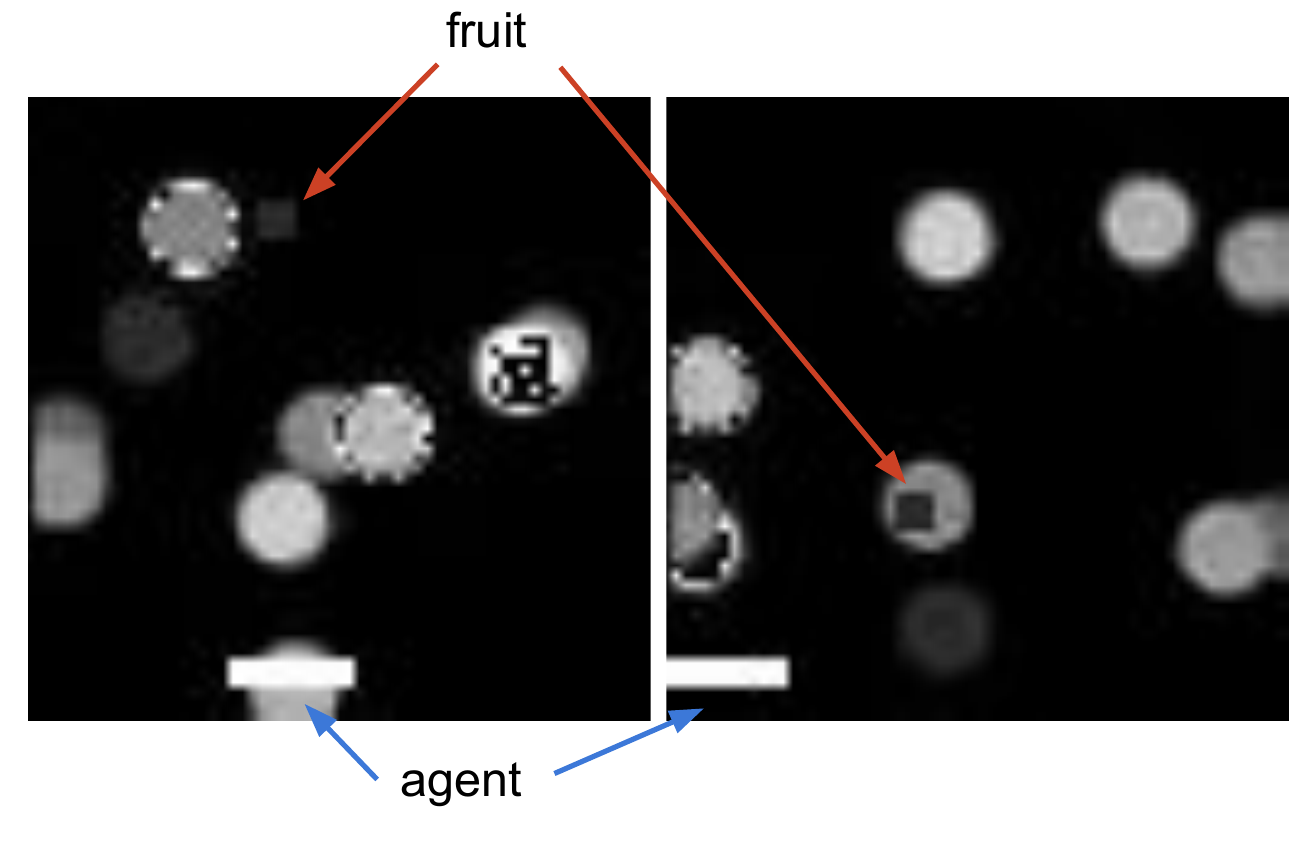}
\caption{Example $64$x$64$ pixel agent observations with background distractors (colored circles) randomly generated at each time step. The distractors increase the difficulty of learning a good representation, which must capture both agent and fruit.}
\label{fig:igr_obs}
\end{wrapfigure}

Analogous to Section~\ref{sec:exps}, in Figure~\ref{fig:results-igr} we show the performance of an RL agent trained with the frozen representation as input (top), as well as the error of decoding true state elements from the representation (bottom). 
In both games, end-to-end RL from images performs poorly, demonstrating the need for representation learning to aid in solving the task.
As predicted by the theory, the representation learned by $\inv$ fails in both games, and the representation learned by $\cpc$ fails in the \textit{catcher-grip} game. 
The difference in performance between sufficient and insufficient objectives is even more pronounced in this setting than in the plain background setting.
With more information present in the observation in the form of the distractors, insufficient objectives that do not optimize for representing all the required state information may be ``distracted'' by representing the background objects instead, resulting in low performance.  
In Appendix~\ref{app:distractors} we experiment with visual distractors that are temporally correlated across time.
We also consider variations on our analysis, evaluating how well the representations predict the predict the optimal $Q^*$ in Appendix~\ref{app:expertq}, and experimenting with a different data distribution for collecting the representation learning dataset in Appendix~\ref{app:data-dist}.

\begin{figure}[!t]
    \centering
          \includegraphics[width=6.5cm, trim=0.25cm 0 0 0, clip]{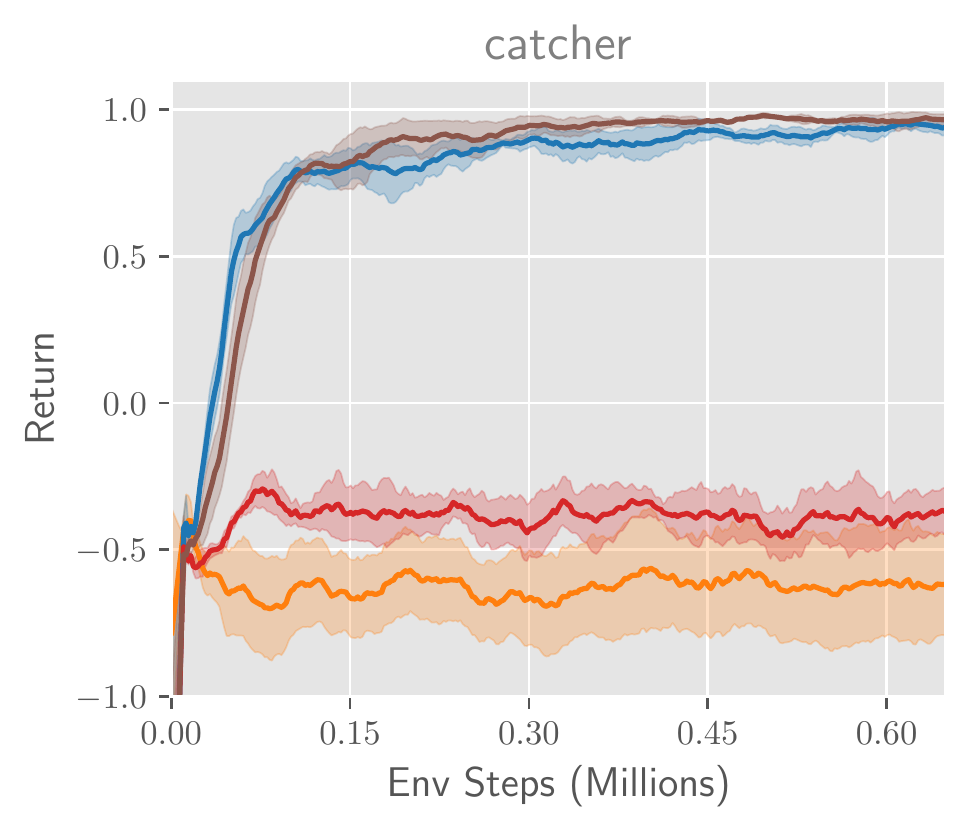} 
          \includegraphics[width=6.5cm, trim=0.25cm 0 0 0, clip]{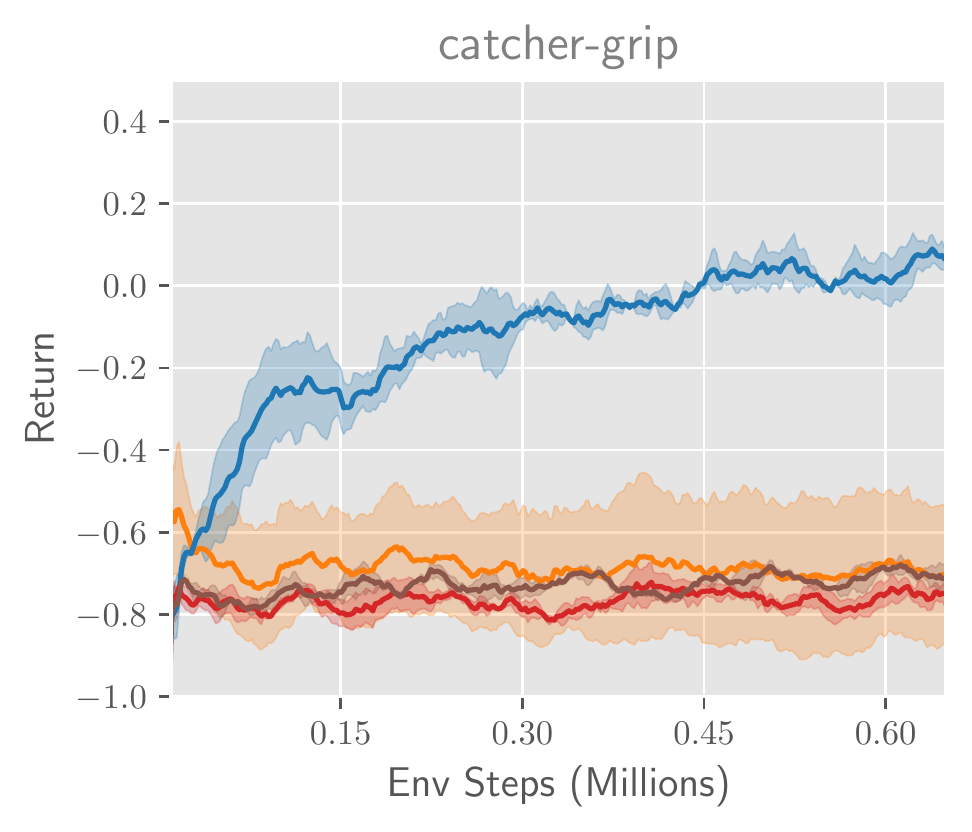} \\
          \includegraphics[height=.90cm]{figs/legend} \\
          \hspace{.5cm}
          \begingroup
          \small
          \begin{tabular}{ ccc } \toprule
            Obj. & Agent Err. & Fruit Err.  \\
            \midrule
            $\fwd$ & 0.13 & 0.14 \\
            \midrule
            $\inv$ & 0.15 & 0.47 \\
            \bottomrule
            \end{tabular}
            \hspace{1cm}
          \begin{tabular}{ cccc }\toprule
            Obj. & Agent Err. & Fruit Err. & Grip Err.  \\
            \midrule
            $\fwd$ & 0.12 & 0.1 & 0.25 \\
            \midrule
            $\cpc$ & 0.1 & 0.14 & 0.5 \\
            \bottomrule
            \end{tabular}
        \endgroup
        \caption{With background distractors added to the observations, the state representation learned via $\inv$ fails to capture the fruit object accurately in \textit{catcher} (left), and the representation learned via $\cpc$ continue to perform poorly at capturing the gripper state in \textit{catcher-grip} (right). The performance of the insufficient representations is even lower than in the clean background experiment.}
\label{fig:results-igr}
\vspace{-3mm}
\end{figure}

\section{Discussion}
\label{sec:discussion}
In this work, we analyze which common MI-based representation learning objectives are guaranteed to yield representations provably sufficient for learning the optimal policy.
We show that two common objectives $\cpc$ and $\inv$ yield theoretically insufficient representations, and provide a proof of sufficiency for $\fwd$. 
We then show that insufficiency of representations can degrade the performance of deep RL agents with experiments on a simulated environment with visual observations.
While an insufficient representation learning objective \emph{can} work well for training RL agents on simulated benchmark environments, the same objective may fail for a real-world system with different characteristics.
We believe that encouraging a focus on evaluating the sufficiency of newly proposed representation learning objectives can help better predict potential failures.

Of the objectives studied in our analysis, only $\fwd$ is sufficient to represent optimal Q-functions for any reward function. 
Note however that the representation obtained by this objective lacks a notion of ``task-relevance'' as it must be equally predictive of all predictable elements in the state. 
Therefore an interesting related question is whether it is possible to construct a general representation that is sufficient for any policy but contains less information than the smallest representation that maximizes $\fwd$. 
We conjecture that the answer is no, but leave investigation of this point for future work.

In addition to this direction, we see many other exciting avenues for future work.
First, identifying more restrictive MDP classes in which insufficient objectives are in fact sufficient.
For example, $\cpc$ is trivially sufficient when the environment dynamics and the agent's policy are deterministic. 
However, we hypothesize there may be more interesting MDP classes, related to realistic applications, in which generally insufficient objectives may be sufficient.
Additionally, extending our analysis to the partially observed setting would be more reflective of practical applications.
An interesting class of models to consider in this context are generative models such as variational auto-encoders~\citep{kingma2013auto}.
Prior work has shown that maximizing the ELBO alone cannot control the content of the learned representation~\citep{huszar2017maximum, phuong2018mutual, alemi2018fixing}. 
We conjecture that the zero-distortion maximizer of the ELBO would be sufficient, while other solutions would not necessarily be.
We see these directions as fruitful in providing a deeper understanding of the learning dynamics of deep RL, and potentially yielding novel algorithms for provably accelerating RL with representation learning.

\textbf{Acknowledgements.} We would like to thank Ignasi Clavera, Chelsea Finn, and Ben Poole for insightful conversations at different stages of this work. 
This research was supported by the DARPA Assured Autonomy program and ARL DCIST CRA W911NF-17-2-0181.

\bibliographystyle{abbrvnat}
\bibliography{references}

\begin{thebibliography}{73}
\providecommand{\natexlab}[1]{#1}
\providecommand{\url}[1]{\texttt{#1}}
\expandafter\ifx\csname urlstyle\endcsname\relax
  \providecommand{\doi}[1]{doi: #1}\else
  \providecommand{\doi}{doi: \begingroup \urlstyle{rm}\Url}\fi

\bibitem[Abel et~al.(2016)Abel, Hershkowitz, and Littman]{abel2016near}
D.~Abel, D.~E. Hershkowitz, and M.~L. Littman.
\newblock Near optimal behavior via approximate state abstraction.
\newblock In \emph{Proceedings of the 33rd International Conference on
  International Conference on Machine Learning-Volume 48}, pages 2915--2923,
  2016.

\bibitem[Agrawal et~al.(2016)Agrawal, Nair, Abbeel, Malik, and
  Levine]{agrawal2016learning}
P.~Agrawal, A.~V. Nair, P.~Abbeel, J.~Malik, and S.~Levine.
\newblock Learning to poke by poking: Experiential learning of intuitive
  physics.
\newblock In \emph{Advances in neural information processing systems}, pages
  5074--5082, 2016.

\bibitem[Alemi et~al.(2018)Alemi, Poole, Fischer, Dillon, Saurous, and
  Murphy]{alemi2018fixing}
A.~A. Alemi, B.~Poole, I.~Fischer, J.~V. Dillon, R.~A. Saurous, and K.~Murphy.
\newblock Fixing a broken elbo.
\newblock In \emph{International Conference on Machine Learning (ICML)}, 2018.

\bibitem[Anand et~al.(2019)Anand, Racah, Ozair, Bengio, C{\^o}t{\'e}, and
  Hjelm]{anand2019unsupervised}
A.~Anand, E.~Racah, S.~Ozair, Y.~Bengio, M.-A. C{\^o}t{\'e}, and R.~D. Hjelm.
\newblock Unsupervised state representation learning in atari.
\newblock In \emph{Advances in Neural Information Processing Systems}, pages
  8766--8779, 2019.

\bibitem[Bachman et~al.(2019)Bachman, Hjelm, and
  Buchwalter]{bachman2019learning}
P.~Bachman, R.~D. Hjelm, and W.~Buchwalter.
\newblock Learning representations by maximizing mutual information across
  views.
\newblock In \emph{Advances in Neural Information Processing Systems}, pages
  15509--15519, 2019.

\bibitem[Bean et~al.(1987)Bean, Birge, and Smith]{bean1987aggregation}
J.~C. Bean, J.~R. Birge, and R.~L. Smith.
\newblock Aggregation in dynamic programming.
\newblock \emph{Operations Research}, 35\penalty0 (2):\penalty0 215--220, 1987.

\bibitem[Becker and Hinton(1992)]{becker1992self}
S.~Becker and G.~E. Hinton.
\newblock Self-organizing neural network that discovers surfaces in random-dot
  stereograms.
\newblock \emph{Nature}, 355\penalty0 (6356):\penalty0 161--163, 1992.

\bibitem[Belghazi et~al.(2018)Belghazi, Baratin, Rajeswar, Ozair, Bengio,
  Courville, and Hjelm]{belghazi2018mine}
M.~I. Belghazi, A.~Baratin, S.~Rajeswar, S.~Ozair, Y.~Bengio, A.~Courville, and
  R.~D. Hjelm.
\newblock Mine: mutual information neural estimation.
\newblock \emph{arXiv preprint arXiv:1801.04062}, 2018.

\bibitem[Bell and Sejnowski(1995)]{bell1995information}
A.~J. Bell and T.~J. Sejnowski.
\newblock An information-maximization approach to blind separation and blind
  deconvolution.
\newblock \emph{Neural computation}, 7\penalty0 (6):\penalty0 1129--1159, 1995.

\bibitem[Bellemare et~al.(2019)Bellemare, Dabney, Dadashi, Taiga, Castro,
  Le~Roux, Schuurmans, Lattimore, and Lyle]{bellemare2019geometric}
M.~Bellemare, W.~Dabney, R.~Dadashi, A.~A. Taiga, P.~S. Castro, N.~Le~Roux,
  D.~Schuurmans, T.~Lattimore, and C.~Lyle.
\newblock A geometric perspective on optimal representations for reinforcement
  learning.
\newblock In \emph{Advances in Neural Information Processing Systems}, pages
  4360--4371, 2019.

\bibitem[Bertsekas et~al.(1988)Bertsekas, Castanon,
  et~al.]{bertsekas1988adaptive}
D.~P. Bertsekas, D.~A. Castanon, et~al.
\newblock Adaptive aggregation methods for infinite horizon dynamic
  programming.
\newblock 1988.

\bibitem[Cover(1999)]{cover1999elements}
T.~M. Cover.
\newblock \emph{Elements of information theory}.
\newblock John Wiley \& Sons, 1999.

\bibitem[Dean and Givan(1997)]{dean1997model}
T.~Dean and R.~Givan.
\newblock Model minimization in markov decision processes.
\newblock In \emph{AAAI/IAAI}, pages 106--111, 1997.

\bibitem[Devlin et~al.(2019)Devlin, Chang, Lee, and Toutanova]{devlin2019bert}
J.~Devlin, M.-W. Chang, K.~Lee, and K.~Toutanova.
\newblock Bert: Pre-training of deep bidirectional transformers for language
  understanding.
\newblock In \emph{Proceedings of the 2019 Conference of the North American
  Chapter of the Association for Computational Linguistics: Human Language
  Technologies, Volume 1 (Long and Short Papers)}, pages 4171--4186, 2019.

\bibitem[Dosovitskiy and Koltun(2016)]{dosovitskiy2016learning}
A.~Dosovitskiy and V.~Koltun.
\newblock Learning to act by predicting the future.
\newblock In \emph{International Conference on Learning Representations}, 2016.

\bibitem[Du et~al.(2019)Du, Kakade, Wang, and Yang]{du2019good}
S.~S. Du, S.~M. Kakade, R.~Wang, and L.~F. Yang.
\newblock Is a good representation sufficient for sample efficient
  reinforcement learning?
\newblock \emph{arXiv preprint arXiv:1910.03016}, 2019.

\bibitem[Fedus et~al.(2019)Fedus, Gelada, Bengio, Bellemare, and
  Larochelle]{fedus2019hyperbolic}
W.~Fedus, C.~Gelada, Y.~Bengio, M.~G. Bellemare, and H.~Larochelle.
\newblock Hyperbolic discounting and learning over multiple horizons.
\newblock \emph{arXiv preprint arXiv:1902.06865}, 2019.

\bibitem[Ferns et~al.(2006)Ferns, Castro, Precup, and
  Panangaden]{ferns2006methods}
N.~Ferns, P.~S. Castro, D.~Precup, and P.~Panangaden.
\newblock Methods for computing state similarity in markov decision processes.
\newblock In \emph{Proceedings of the Twenty-Second Conference on Uncertainty
  in Artificial Intelligence}, pages 174--181, 2006.

\bibitem[Finn et~al.(2016)Finn, Tan, Duan, Darrell, Levine, and
  Abbeel]{finn2016deep}
C.~Finn, X.~Y. Tan, Y.~Duan, T.~Darrell, S.~Levine, and P.~Abbeel.
\newblock Deep spatial autoencoders for visuomotor learning.
\newblock In \emph{2016 IEEE International Conference on Robotics and
  Automation (ICRA)}, pages 512--519. IEEE, 2016.

\bibitem[garage contributors(2019)]{garage}
T.~garage contributors.
\newblock Garage: A toolkit for reproducible reinforcement learning research.
\newblock \url{https://github.com/rlworkgroup/garage}, 2019.

\bibitem[Gelada et~al.(2019)Gelada, Kumar, Buckman, Nachum, and
  Bellemare]{gelada2019deepmdp}
C.~Gelada, S.~Kumar, J.~Buckman, O.~Nachum, and M.~G. Bellemare.
\newblock Deepmdp: Learning continuous latent space models for representation
  learning.
\newblock In \emph{International Conference on Machine Learning}, pages
  2170--2179, 2019.

\bibitem[Givan et~al.(2003)Givan, Dean, and Greig]{givan2003equivalence}
R.~Givan, T.~Dean, and M.~Greig.
\newblock Equivalence notions and model minimization in markov decision
  processes.
\newblock \emph{Artificial Intelligence}, 147\penalty0 (1-2):\penalty0
  163--223, 2003.

\bibitem[Gregor et~al.(2016)Gregor, Rezende, and
  Wierstra]{gregor2016variational}
K.~Gregor, D.~J. Rezende, and D.~Wierstra.
\newblock Variational intrinsic control.
\newblock \emph{arXiv preprint arXiv:1611.07507}, 2016.

\bibitem[Gutmann and Hyv{\"a}rinen(2010)]{gutmann2010noise}
M.~Gutmann and A.~Hyv{\"a}rinen.
\newblock Noise-contrastive estimation: A new estimation principle for
  unnormalized statistical models.
\newblock In \emph{Proceedings of the Thirteenth International Conference on
  Artificial Intelligence and Statistics}, pages 297--304, 2010.

\bibitem[Haarnoja et~al.(2018)Haarnoja, Zhou, Abbeel, and
  Levine]{haarnoja2018soft}
T.~Haarnoja, A.~Zhou, P.~Abbeel, and S.~Levine.
\newblock Soft actor-critic: Off-policy maximum entropy deep reinforcement
  learning with a stochastic actor.
\newblock \emph{arXiv preprint arXiv:1801.01290}, 2018.

\bibitem[Hafner et~al.(2019)Hafner, Lillicrap, Fischer, Villegas, Ha, Lee, and
  Davidson]{hafner2019learning}
D.~Hafner, T.~Lillicrap, I.~Fischer, R.~Villegas, D.~Ha, H.~Lee, and
  J.~Davidson.
\newblock Learning latent dynamics for planning from pixels.
\newblock In \emph{International Conference on Machine Learning}, pages
  2555--2565, 2019.

\bibitem[Hjelm et~al.(2018)Hjelm, Fedorov, Lavoie-Marchildon, Grewal,
  Trischler, and Bengio]{hjelm2018learning}
R.~D. Hjelm, A.~Fedorov, S.~Lavoie-Marchildon, K.~Grewal, A.~Trischler, and
  Y.~Bengio.
\newblock Learning deep representations by mutual information estimation and
  maximization.
\newblock \emph{arXiv preprint arXiv:1808.06670}, 2018.

\bibitem[Husz{\'a}r(2017)]{huszar2017maximum}
F.~Husz{\'a}r.
\newblock Is maximum likelihood useful for representation learning?, 2017.

\bibitem[Jaderberg et~al.(2016)Jaderberg, Mnih, Czarnecki, Schaul, Leibo,
  Silver, and Kavukcuoglu]{jaderberg2016reinforcement}
M.~Jaderberg, V.~Mnih, W.~M. Czarnecki, T.~Schaul, J.~Z. Leibo, D.~Silver, and
  K.~Kavukcuoglu.
\newblock Reinforcement learning with unsupervised auxiliary tasks.
\newblock \emph{arXiv preprint arXiv:1611.05397}, 2016.

\bibitem[Jong and Stone(2005)]{jong2005state}
N.~K. Jong and P.~Stone.
\newblock State abstraction discovery from irrelevant state variables.
\newblock In \emph{IJCAI}, volume~8, pages 752--757, 2005.

\bibitem[Jonschkowski and Brock(2015)]{jonschkowski2015learning}
R.~Jonschkowski and O.~Brock.
\newblock Learning state representations with robotic priors.
\newblock \emph{Autonomous Robots}, 39\penalty0 (3):\penalty0 407--428, 2015.

\bibitem[Kalashnikov et~al.(2018)Kalashnikov, Irpan, Pastor, Ibarz, Herzog,
  Jang, Quillen, Holly, Kalakrishnan, Vanhoucke,
  et~al.]{kalashnikov2018scalable}
D.~Kalashnikov, A.~Irpan, P.~Pastor, J.~Ibarz, A.~Herzog, E.~Jang, D.~Quillen,
  E.~Holly, M.~Kalakrishnan, V.~Vanhoucke, et~al.
\newblock Scalable deep reinforcement learning for vision-based robotic
  manipulation.
\newblock In \emph{Conference on Robot Learning}, pages 651--673, 2018.

\bibitem[Karl et~al.(2016)Karl, Soelch, Bayer, and van~der Smagt]{karl2016deep}
M.~Karl, M.~Soelch, J.~Bayer, and P.~van~der Smagt.
\newblock Deep variational bayes filters: Unsupervised learning of state space
  models from raw data.
\newblock In \emph{International Conference on Learning Representations}, 2016.

\bibitem[Kingma and Welling(2013)]{kingma2013auto}
D.~P. Kingma and M.~Welling.
\newblock Auto-encoding variational bayes.
\newblock \emph{arXiv preprint arXiv:1312.6114}, 2013.

\bibitem[Kipf et~al.(2019)Kipf, van~der Pol, and Welling]{kipf2019contrastive}
T.~Kipf, E.~van~der Pol, and M.~Welling.
\newblock Contrastive learning of structured world models.
\newblock \emph{arXiv preprint arXiv:1911.12247}, 2019.

\bibitem[Klyubin et~al.(2005)Klyubin, Polani, and
  Nehaniv]{klyubin2005empowerment}
A.~S. Klyubin, D.~Polani, and C.~L. Nehaniv.
\newblock Empowerment: A universal agent-centric measure of control.
\newblock In \emph{2005 IEEE Congress on Evolutionary Computation}, volume~1,
  pages 128--135. IEEE, 2005.

\bibitem[Klyubin et~al.(2008)Klyubin, Polani, and Nehaniv]{klyubin2008keep}
A.~S. Klyubin, D.~Polani, and C.~L. Nehaniv.
\newblock Keep your options open: An information-based driving principle for
  sensorimotor systems.
\newblock \emph{PloS one}, 3\penalty0 (12), 2008.

\bibitem[Lattimore and Szepesvari(2019)]{lattimore2019learning}
T.~Lattimore and C.~Szepesvari.
\newblock Learning with good feature representations in bandits and in rl with
  a generative model.
\newblock \emph{arXiv preprint arXiv:1911.07676}, 2019.

\bibitem[Lee et~al.(2019)Lee, Nagabandi, Abbeel, and Levine]{lee2019slac}
A.~X. Lee, A.~Nagabandi, P.~Abbeel, and S.~Levine.
\newblock Stochastic latent actor-critic: Deep reinforcement learning with a
  latent variable model.
\newblock \emph{arXiv preprint arXiv:1907.00953}, 2019.

\bibitem[Leibfried et~al.(2019)Leibfried, Pascual-D{\'\i}az, and
  Grau-Moya]{leibfried2019unified}
F.~Leibfried, S.~Pascual-D{\'\i}az, and J.~Grau-Moya.
\newblock A unified bellman optimality principle combining reward maximization
  and empowerment.
\newblock In \emph{Advances in Neural Information Processing Systems}, pages
  7867--7878, 2019.

\bibitem[Lesort et~al.(2018)Lesort, D{\'\i}az-Rodr{\'\i}guez, Goudou, and
  Filliat]{lesort2018state}
T.~Lesort, N.~D{\'\i}az-Rodr{\'\i}guez, J.-F. Goudou, and D.~Filliat.
\newblock State representation learning for control: An overview.
\newblock \emph{Neural Networks}, 108:\penalty0 379--392, 2018.

\bibitem[Li et~al.(2006)Li, Walsh, and Littman]{li2006towards}
L.~Li, T.~J. Walsh, and M.~L. Littman.
\newblock Towards a unified theory of state abstraction for mdps.
\newblock In \emph{International Symposium on Artificial Intelligence and
  Mathematics, {ISAIM} 2006, Fort Lauderdale, Florida, USA, January 4-6, 2006},
  2006.
\newblock URL \url{http://anytime.cs.umass.edu/aimath06/proceedings/P21.pdf}.

\bibitem[Linsker(1988)]{linsker1988self}
R.~Linsker.
\newblock Self-organization in a perceptual network.
\newblock \emph{Computer}, 21\penalty0 (3):\penalty0 105--117, 1988.

\bibitem[Mazoure et~al.(2020)Mazoure, Combes, Doan, Bachman, and
  Hjelm]{mazoure2020deep}
B.~Mazoure, R.~T.~d. Combes, T.~Doan, P.~Bachman, and R.~D. Hjelm.
\newblock Deep reinforcement and infomax learning.
\newblock \emph{arXiv preprint arXiv:2006.07217}, 2020.

\bibitem[McAllester and Statos(2018)]{mcallester2018formal}
D.~McAllester and K.~Statos.
\newblock Formal limitations on the measurement of mutual information.
\newblock \emph{arXiv preprint arXiv:1811.04251}, 2018.

\bibitem[McCallum(1996)]{mccallum1996reinforcement}
A.~K. McCallum.
\newblock \emph{Reinforcement Learning with Selective Perception and Hidden
  State}.
\newblock PhD thesis, University of Rochester, 1996.

\bibitem[Mirowski(2019)]{mirowski2019learning}
P.~Mirowski.
\newblock Learning to navigate.
\newblock In \emph{1st International Workshop on Multimodal Understanding and
  Learning for Embodied Applications}, pages 25--25, 2019.

\bibitem[Mnih et~al.(2013)Mnih, Kavukcuoglu, Silver, Graves, Antonoglou,
  Wierstra, and Riedmiller]{mnih2013playing}
V.~Mnih, K.~Kavukcuoglu, D.~Silver, A.~Graves, I.~Antonoglou, D.~Wierstra, and
  M.~Riedmiller.
\newblock Playing atari with deep reinforcement learning.
\newblock \emph{arXiv preprint arXiv:1312.5602}, 2013.

\bibitem[Mohamed and Rezende(2015)]{mohamed2015variational}
S.~Mohamed and D.~J. Rezende.
\newblock Variational information maximisation for intrinsically motivated
  reinforcement learning.
\newblock In \emph{Advances in neural information processing systems}, pages
  2125--2133, 2015.

\bibitem[Nachum et~al.(2018)Nachum, Gu, Lee, and Levine]{nachum2018near}
O.~Nachum, S.~Gu, H.~Lee, and S.~Levine.
\newblock Near-optimal representation learning for hierarchical reinforcement
  learning.
\newblock 2018.

\bibitem[Oord et~al.(2018)Oord, Li, and Vinyals]{oord2018representation}
A.~v.~d. Oord, Y.~Li, and O.~Vinyals.
\newblock Representation learning with contrastive predictive coding.
\newblock \emph{arXiv preprint arXiv:1807.03748}, 2018.

\bibitem[Pathak et~al.(2017)Pathak, Agrawal, Efros, and
  Darrell]{pathak2017curiosity}
D.~Pathak, P.~Agrawal, A.~A. Efros, and T.~Darrell.
\newblock Curiosity-driven exploration by self-supervised prediction.
\newblock In \emph{Proceedings of the 34th International Conference on Machine
  Learning-Volume 70}, pages 2778--2787. JMLR. org, 2017.

\bibitem[Phuong et~al.(2018)Phuong, Welling, Kushman, Tomioka, and
  Nowozin]{phuong2018mutual}
M.~Phuong, M.~Welling, N.~Kushman, R.~Tomioka, and S.~Nowozin.
\newblock The mutual autoencoder: Controlling information in latent code
  representations.
\newblock 2018.

\bibitem[Poole et~al.(2019)Poole, Ozair, Oord, Alemi, and
  Tucker]{poole2019variational}
B.~Poole, S.~Ozair, A.~v.~d. Oord, A.~A. Alemi, and G.~Tucker.
\newblock On variational bounds of mutual information.
\newblock \emph{arXiv preprint arXiv:1905.06922}, 2019.

\bibitem[Ravindran and Barto(2003)]{ravindran2003smdp}
B.~Ravindran and A.~G. Barto.
\newblock Smdp homomorphisms: an algebraic approach to abstraction in
  semi-markov decision processes.
\newblock In \emph{Proceedings of the 18th international joint conference on
  Artificial intelligence}, pages 1011--1016, 2003.

\bibitem[Schwarzer et~al.()Schwarzer, Anand, Goel, Hjelm, Courville, and
  Bachman]{schwarzerdata}
M.~Schwarzer, A.~Anand, R.~Goel, R.~D. Hjelm, A.~Courville, and P.~Bachman.
\newblock Data-efficient reinforcement learning with self-predictive
  representations.

\bibitem[Shelhamer et~al.(2016)Shelhamer, Mahmoudieh, Argus, and
  Darrell]{shelhamer2016loss}
E.~Shelhamer, P.~Mahmoudieh, M.~Argus, and T.~Darrell.
\newblock Loss is its own reward: Self-supervision for reinforcement learning.
\newblock \emph{arXiv preprint arXiv:1612.07307}, 2016.

\bibitem[Shinners(2011)]{pygame}
P.~Shinners.
\newblock Pygame.
\newblock \url{http://pygame.org/}, 2011.

\bibitem[Shu et~al.(2020)Shu, Nguyen, Chow, Pham, Than, Ghavamzadeh, Ermon, and
  Bui]{shu2020predictive}
R.~Shu, T.~Nguyen, Y.~Chow, T.~Pham, K.~Than, M.~Ghavamzadeh, S.~Ermon, and
  H.~Bui.
\newblock Predictive coding for locally-linear control.
\newblock In \emph{International Conference on Machine Learning}, pages
  8862--8871. PMLR, 2020.

\bibitem[Song and Ermon(2019)]{song2019understanding}
J.~Song and S.~Ermon.
\newblock Understanding the limitations of variational mutual information
  estimators.
\newblock \emph{arXiv preprint arXiv:1910.06222}, 2019.

\bibitem[Srinivas et~al.(2020)Srinivas, Laskin, and Abbeel]{srinivas2020curl}
A.~Srinivas, M.~Laskin, and P.~Abbeel.
\newblock Curl: Contrastive unsupervised representations for reinforcement
  learning.
\newblock \emph{arXiv preprint arXiv:2004.04136}, 2020.

\bibitem[Still and Precup(2012)]{still2012information}
S.~Still and D.~Precup.
\newblock An information-theoretic approach to curiosity-driven reinforcement
  learning.
\newblock \emph{Theory in Biosciences}, 131\penalty0 (3):\penalty0 139--148,
  2012.

\bibitem[Stooke et~al.(2020)Stooke, Lee, Abbeel, and
  Laskin]{stooke2020decoupling}
A.~Stooke, K.~Lee, P.~Abbeel, and M.~Laskin.
\newblock Decoupling representation learning from reinforcement learning.
\newblock Technical report, UC Berkeley, 2020.

\bibitem[Sun et~al.(2019)Sun, Baradel, Murphy, and Schmid]{sun2019contrastive}
C.~Sun, F.~Baradel, K.~Murphy, and C.~Schmid.
\newblock Contrastive bidirectional transformer for temporal representation
  learning.
\newblock \emph{arXiv preprint arXiv:1906.05743}, 2019.

\bibitem[Sutton and Barto(2018)]{sutton2018reinforcement}
R.~S. Sutton and A.~G. Barto.
\newblock \emph{Reinforcement learning: An introduction}.
\newblock MIT press, 2018.

\bibitem[Thomas et~al.(2017)Thomas, Pondard, Bengio, Sarfati, Beaudoin, Meurs,
  Pineau, Precup, and Bengio]{thomas2017independently}
V.~Thomas, J.~Pondard, E.~Bengio, M.~Sarfati, P.~Beaudoin, M.-J. Meurs,
  J.~Pineau, D.~Precup, and Y.~Bengio.
\newblock Independently controllable factors.
\newblock \emph{arXiv preprint arXiv:1708.01289}, 2017.

\bibitem[Van~Roy and Dong(2019)]{van2019comments}
B.~Van~Roy and S.~Dong.
\newblock Comments on the du-kakade-wang-yang lower bounds.
\newblock \emph{arXiv preprint arXiv:1911.07910}, 2019.

\bibitem[Veeriah et~al.(2019)Veeriah, Hessel, Xu, Rajendran, Lewis, Oh, van
  Hasselt, Silver, and Singh]{veeriah2019discovery}
V.~Veeriah, M.~Hessel, Z.~Xu, J.~Rajendran, R.~L. Lewis, J.~Oh, H.~P. van
  Hasselt, D.~Silver, and S.~Singh.
\newblock Discovery of useful questions as auxiliary tasks.
\newblock In \emph{Advances in Neural Information Processing Systems}, pages
  9306--9317, 2019.

\bibitem[Watter et~al.(2015)Watter, Springenberg, Boedecker, and
  Riedmiller]{watter2015embed}
M.~Watter, J.~Springenberg, J.~Boedecker, and M.~Riedmiller.
\newblock Embed to control: A locally linear latent dynamics model for control
  from raw images.
\newblock In \emph{Advances in neural information processing systems}, pages
  2746--2754, 2015.

\bibitem[Yu et~al.(2019)Yu, Shevchuk, Sadigh, and Finn]{yuunsupervised}
T.~Yu, G.~Shevchuk, D.~Sadigh, and C.~Finn.
\newblock Unsupervised visuomotor control through distributional planning
  networks.
\newblock In \emph{Robotics Science and Systems}, 2019.

\bibitem[Zhan et~al.(2020)Zhan, Zhao, Pinto, Abbeel, and
  Laskin]{zhan2020framework}
A.~Zhan, P.~Zhao, L.~Pinto, P.~Abbeel, and M.~Laskin.
\newblock A framework for efficient robotic manipulation.
\newblock \emph{arXiv preprint arXiv:2012.07975}, 2020.

\bibitem[Zhang et~al.(2018{\natexlab{a}})Zhang, Satija, and
  Pineau]{zhang2018decoupling}
A.~Zhang, H.~Satija, and J.~Pineau.
\newblock Decoupling dynamics and reward for transfer learning.
\newblock \emph{arXiv preprint arXiv:1804.10689}, 2018{\natexlab{a}}.

\bibitem[Zhang et~al.(2018{\natexlab{b}})Zhang, Vikram, Smith, Abbeel, Johnson,
  and Levine]{zhang2018solar}
M.~Zhang, S.~Vikram, L.~Smith, P.~Abbeel, M.~J. Johnson, and S.~Levine.
\newblock Solar: Deep structured latent representations for model-based
  reinforcement learning.
\newblock \emph{arXiv preprint arXiv:1808.09105}, 2018{\natexlab{b}}.

\end{thebibliography}
\appendix
\newpage
\onecolumn
\section{Appendix}
\subsection{Sufficiency of $\fwd$: Proof of Proposition~\ref{prop:fwd_suff}}
\label{app:proof}
\newcommand{\Qs}{\mathcal{R}_{\vals, \vala}}
\newcommand{\Qz}{\mathcal{R}_{\valz, \vala}}
We describe the proofs for the sufficiency results from Section \ref{sec:sufficiency} here. We begin by providing a set of lemmas, before proving the sufficiency of $\fwd$. 

\begin{figure*}[!h]
  \centering
  \includegraphics[width=5cm]{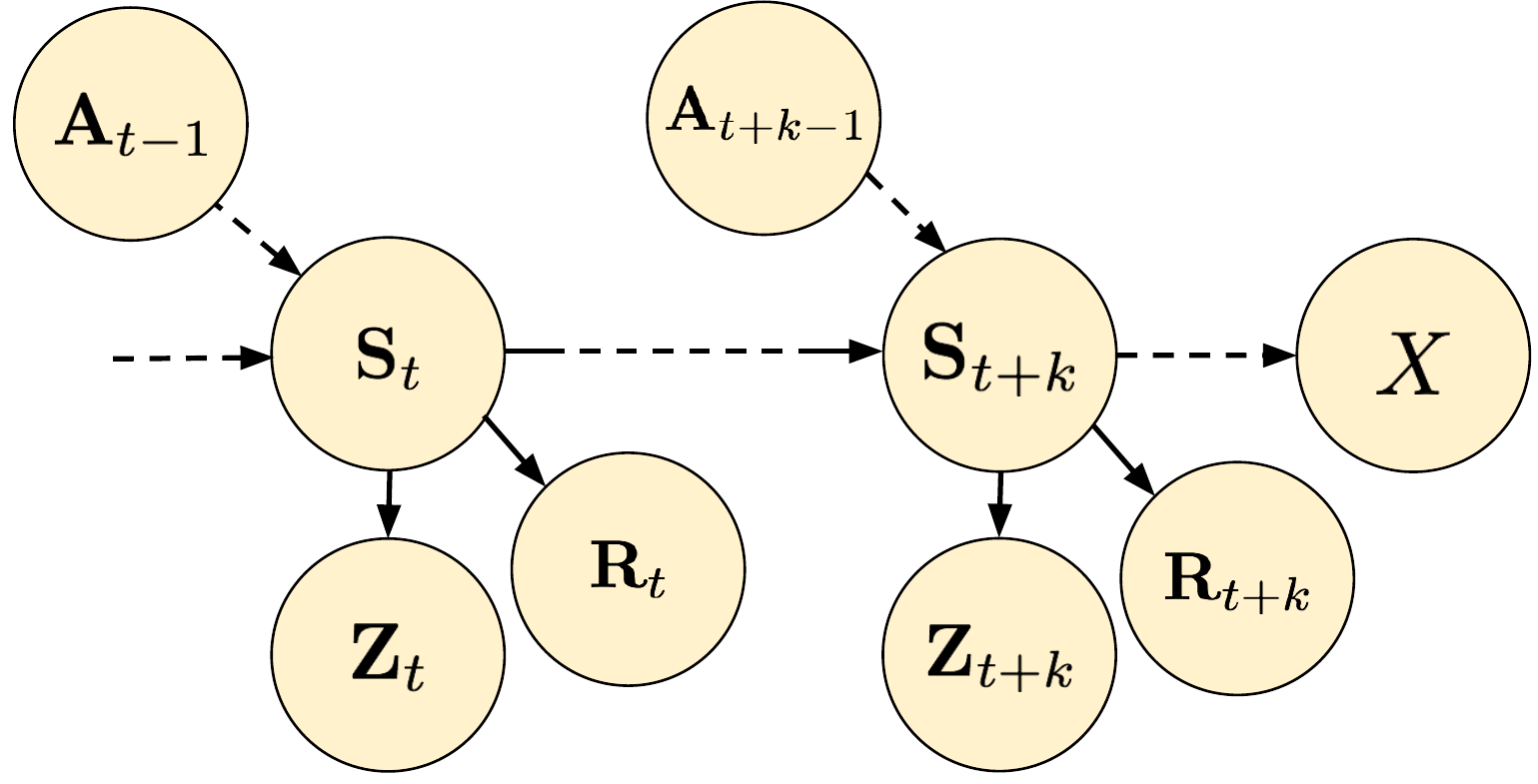}
  \caption{Graphical model for Lemma~\ref{lemma:1}, depicting true states $\vars$, states in the representation $\varz$, actions $\vara$, rewards $\varr$, and the variable $X$ (which we will interpret as the sum of future rewards in the proof of Proposition~\ref{prop:fwd_suff}).}
  \label{fig:proof_pgm}%
\end{figure*}

\begin{lemma}
Let $X$ be a random variable dependent on $\vars_{t+k}$, with the conditional independence assumptions implied by the graphical model in Figure~\ref{fig:proof_pgm}. (In the main proof of Proposition~\ref{prop:fwd_suff}, we will let $X$ be the sum of rewards from time $t+k$ onwards.) If $I(\varz_{t+k}; \varz_t, \vara_t) = I(\vars_{t+k}; \vars_t, \vara_t) \forall k$, then $I(X; \varz_t, \vara_t) = I(X; \vars_t, \vara_t) \forall k$. 
\label{lemma:1}
\end{lemma}

\newcommand{\s}{\vars_t}
\newcommand{\ns}{\vars_{t+k}}
\newcommand{\z}{\varz_t}
\newcommand{\nz}{\varz_{t+k}}
\newcommand{\nr}{X}
\newcommand{\act}{\vara_t}
\begin{proof}
Recall from Definition~\ref{def:rep} that $\rep(\vals) = p(\varz|\vars=\vals)$. For proof by contradiction, assume there is some $\rep$ and some $X$ such that $I(\nr; \z, \act) < I(\nr; \s, \act)$ and that $I(\nz; \z, \act) = I(\ns; \s, \act)$.
Now we know that because $\z \rightarrow \s \rightarrow \ns \rightarrow \nz$ form a Markov chain, by the data processing inequality (DPI) $I(\nz; \z, \act) \le I(\ns; \z, \act) \le I(\ns; \s, \act)$. 
We will proceed by showing that that $I(\nr; \z, \act) < I(\nr; \s, \act) \implies I(\ns; \z, \act) < I(\ns; \s, \act)$, which gives the needed contradiction.

Using chain rule, we can expand the following expression in two different ways.
\begin{equation}
    I(\nr; \z, \s, \act) = I(\nr; \z | \s, \act) + I(\nr; \s, \act) \\
    = 0 + I(\nr; \s, \act)
\label{eq:l1e1}
\end{equation}

\begin{equation}
    I(\nr; \z, \s, \act) = I(\nr; \s | \z, \act) + I(\nr; \z, \act)
\label{eq:l1e2}
\end{equation}

Note that the first term in Equation~\ref{eq:l1e1} is zero by the conditional independence assumptions in Figure~\ref{fig:proof_pgm}. Equating the expansions, we can see that to satisfy our assumption that $I(\nr; \z, \act) < I(\nr; \s, \act)$, we must have that $I(\nr; \s | \z, \act) > 0$. 

Now we follow a similar procedure to expand the following expression:
\begin{equation}
I(\ns; \z, \s, \act) = I(\ns; \z | \s, \act) + I(\ns; \s, \act) 
= 0 + I(\ns; \s, \act) 
\label{eq:l1e3}
\end{equation}

\begin{equation}
I(\ns; \z, \s, \act) = I(\ns; \s | \z, \act) + I(\ns; \z, \act)
\label{eq:l1e4}
\end{equation}

The first term in Equation~\ref{eq:l1e3} is zero by the conditional independence assumptions in Figure~\ref{fig:proof_pgm}. 
Comparing the first term in Equation~\ref{eq:l1e4} with the first term in Equation~\ref{eq:l1e2}, we see that because $\s \rightarrow \ns \rightarrow \nr$ form a Markov chain, by the DPI $I(\ns; \s | \z, \act) \ge I(\nr; \s | \z, \act)$.
Therefore we must have $I(\ns; \s | \z, \act) > 0$.
Combining Equations~\ref{eq:l1e3} and~\ref{eq:l1e4}:
\begin{equation}
I(\ns; \s, \act) = I(\ns; \s | \z, \act) + I(\ns; \z, \act) 
\end{equation}
Since $I(\ns; \s | \z, \act) > 0$, $I(\ns; \z, \act) < I(\ns; \s, \act)$, which is exactly the contradiction we set out to show.
\end{proof}

\newpage
\begin{lemma}
If $I(Y; Z) = I(Y; X)$ and $Y \perp Z | X$, then $\forall x, p(Y | X = x) = \E_{p(Y | Z)} p(Z | X=x)$.
\label{lemma:2}
\end{lemma}

\begin{proof}
First note that the statement is not trivially true. Without any assumption regarding MI, we can write,
\begin{equation}
p(Y | X = x) = \int p(Y, Z | X=x) dz = \int p(Y | Z, X = x) p(Z | X = x) dz
\end{equation}
Comparing this with the statement we'd like to prove, we can see that the key idea is to show that the MI equivalence implies that $p(Y | Z, X = x) = p(Y | Z)$.

Another way to write what we want to show is
\begin{equation}
    D_{KL}[p(Y|X) || \E_{p(Z|X)} p(Y|Z)] = 0
\end{equation}
Our strategy will be to upper-bound this KL by a quantity that we will show to be $0$. Since KL divergences are always lower-bounded $0$, this will prove equality.
We begin by writing out the definition of the KL divergence, and then using Jensen's inequality to upper-bound it.
\begin{equation}
\begin{split}
    D_{KL}[p(Y|X) || \E_{p(Z|X)} p(Y|Z)] &= \E_{p(Y|X)} \log \left[ \frac{p(Y|X)}{\E_{p(Z|X)}p(Y|Z)} \right] \\
    &= \E_{p(Y|X)} [ \log p(Y|X) - \log \E_{p(Z|X)} p(Y|Z) ] \\
    &\le \E_{p(Y|X)} [ \log p(Y|X) - \E_{p(Z|X)} \log p(Y|Z) ]
\end{split}
\label{eq:l2e1}
\end{equation}

where the last inequality follows by Jensen's inequality. Because KL divergences are greater than $0$, this last expression is also greater than $0$.
Now let us try to relate $I(Y; Z) = I(Y; X)$ to this expression. We can re-write this equality using the entropy definition of MI.
\begin{equation}
H(Y) - H(Y | Z) = H(Y) - H(Y | X)
\end{equation}
The $H(Y)$ cancel and substituting the definition of entropy we have:
\begin{equation}
\E_{p(Y, Z)}  \log p(Y | Z) = \E_{p(Y, X)}  \log p(Y | X)
\end{equation}
On the right-hand side, we can use the Tower property to re-write the expectation as
\begin{equation}
\E_{p(Y, X)}  \log p(Y | X) = \E_{p(Z)} \E_{p(Y, X | Z)}  \log p(Y | X) = \E_{p(Y | X)p(X , Z)}  \log p(Y | X)
\end{equation}
Now we can use the Tower property again to re-write the expectation on both sides.
\begin{equation}
\begin{split}
\E_{p(X)}  \E_{p(Y, Z | X)}  \log p(Y | Z) &= \E_{p(X)}  \E_{p(Y| X)p(X, Z | X)}  \log p(Y | X) \\
\E_{p(X)}  \E_{p(Y | X) p(Z | X)}  \log p(Y | Z) &= \E_{p(X)}  \E_{p(Y | X) p(Z | X)}  \log p(Y | X) \\
\E_{p(X)}  \E_{p(Y | X) p(Z | X)}  [\log p(Y | X) &-  \log p(Y | Z)] = 0 \\
\E_{p(X)}  \E_{p(Y | X)}  [\log p(Y | X) &-  \E_{ p(Z | X)}\log p(Y | Z)] = 0 \\
\end{split}
\label{eq:l2e2}
\end{equation}

Compare this expression to Equation~\ref{eq:l2e1}. Because Equation~\ref{eq:l2e1} is greater than $0$, each term inside $\E_{p(X)}$ in Equation~\ref{eq:l2e2} is the same sign. If the sum of elements that all have the same sign is zero, then each element is zero. Therefore,
\begin{equation}
    0 = \E_{p(Y|X)} [ \log p(Y|X) - \E_{p(Z|X)} \log p(Y|Z) ] \ge D_{KL}[p(Y|X) || \E_{p(Z|X)} p(Y|Z)]
\end{equation}
Since the KL divergence is upper and lower-bounded by $0$, it must be equal to $0$, and therefore $p(Y | X) = \E_{p(Y | Z)} p(Z | X)$ as we wanted to show.
\end{proof}

Given the lemmas stated above, we can then use them to prove the sufficiency of $\fwd$. 

\newpage
\begin{proposition}
 
(Sufficiency of $\fwd$) Let $(\sets, \seta, \mathcal{T})$ be an MDP with dynamics $p(\vars_{t+1} | \vars_t, \vara_t)$. Let the policy distribution $p(\vara | \vars)$ and steady-state state occupancy $p(\vars)$ have full support on the action and state alphabets $\seta$ and $\sets$ respectively. Let $\setr$ be the set of all reward functions that can be expressed as a function of the state, $r: \sets \rightarrow \mathbb{R}$.~\footnote{If the reward is a function of the action, then $\fwd$ may not be sufficient. To see this, first consider an MDP where the reward depends on a random variable in the state that is independently sampled at each timestep. In this case, $\fwd$ is still sufficient, because all policies will be equally poor since the reward at the next timestep is completely unpredictable. However, if the reward depends on the \emph{action} taken, then a policy trained on the original state space could choose the high-reward action, but a representation learned by $\fwd$ needn't capture that information if it doesn't improve prediction of the next state. So in this latter case, $\fwd$ can be insufficient.} See Figure~\ref{fig:proof_pgm} for a graphical depiction of the conditional independence relationships between variables.

For a representation $\rep$, if $I(\varz_{t+1}; \varz_t, \vara_t)$
is maximized $\forall t>0$ then $\forall \valr \in \setr$ and $\forall \vals_1, \vals_2 \in \sets$,  $\rep(\vals_1) = \rep(\vals_2) \implies \forall \vala \text{,} Q^*_{r}(\vala , \vals_1) = Q^*_{r}(\vala , \vals_2)$.
\end{proposition}

\begin{proof}
Note that $(\varz_{t+1}; \varz_t, \vara_t)$ is maximized if the representation $\phi_{\setz}$ is taken to be the identity. In other words $ \max_{\phi} I(\varz_{t+1}; \varz_t, \vara_t) = I(\vars_{t+1}; \vars_t, \vara_t)$. 

Define the random variable $\return_t$ to be the discounted return starting from state $\vals_t$.
\begin{equation}
    \return_t = \sum_{i=1}^{H-t} \gamma^k \varr_{t+i}
\end{equation}
Plug in $\return_{t}$ for the random variable $X$ in Lemma~\ref{lemma:1}:
\begin{equation}
I(\varz_{t+1}; \varz_t, \vara_t) = I(\vars_{t+1}; \vars_t, \vara_t) \qquad  \implies \qquad I(\return_{t+1}; \varz_t, \vara_t) = I(\return_{t+1}; \vars_t, \vara_t)
\end{equation}

Now let $X = [\vars_t, \vara_t]$, $Y = \return_t$, and $Z = \varz_t$, and note that by the structure of the graphical model in Figure~\ref{fig:proof_pgm}, $Y \perp Z | X$. Plugging into Lemma~\ref{lemma:2}:

\begin{equation}
\E_{p(\valz_t | \vars_t = \vals)} p(\return_t | \varz_t, \vara_t) = p(\return_t | \vars_t = \vals, \vara_t)
\label{eq:exp_rbar}
\end{equation}

Now the $Q$-function given a reward function $r$ and a state-action pair $(\vals, \vala)$ can be written as an expectation of this random variable $\return_t$, given $\vars_t = \vals$ and $\vara = \vala$. (Note that $p(\return_t | \vars_t = \vals, \vara_t = \vala)$ can be calculated from the dynamics, policy, and reward distributions.) 

\begin{equation}
Q_r(\vals, \vala) = \E_{p(\return_t | \vars_t = \vals, \vara_t = \vala)} [\return_t] 
\label{eq:Q}
\end{equation}

Since $\rep(\vals_1)=\rep(\vals_2)$, $p(\valz_t | \vars_t=\vals_1) = p(\valz_t | \vars_t = \vals_2)$. Therefore by Equation~\ref{eq:exp_rbar}, $p(\return_t | \vars_t = \vals_1, \vara_t) = p(\return_t | \vars_t = \vals_2, \vara_t)$.
Plugging this result into Equation~\ref{eq:Q}, $Q_{r}(\vala, \vals_1) = Q_{r}(\vala, \vals_2)$.
Because this reasoning holds for all $Q$-functions~\footnote{Note this result is stronger than what we needed: it means that representations that maximize $\fwd$ are guaranteed to be able to represent even sub-optimal $Q$-functions. This makes sense in light of the fact that the proof holds for all reward functions - the sub-optimal $Q$ under one reward is the optimal $Q$ under another.}, it also holds for the optimal $Q$, therefore $Q^*_{r}(\vala, \vals_1) = Q^*_{r}(\vala, \vals_2)$.  

\end{proof}

\subsection{Experimental Details and Further Experiments}
\label{app:exp}

\subsubsection{Didactic Experiments}
The didactic examples are computed as follows.
Given the list of states in the MDP, we compute the possible representations, restricting our focus to representations that group states into “blocks.”  
We do this because there are infinite stochastic representations and the MI expressions we consider are not convex in the parameters of $p(Z | S)$, making searching over these representations difficult. 
Given each state representation, we compute the value of the MI objective as well as the optimal value function using exact value iteration.
In these examples, we assume that the policy distribution is uniform, and that the environment dynamics are deterministic. 
Since we consider the infinite horizon setting, we use the steady-state state occupancy in our calculations.

\subsubsection{Deep RL Experiments}
\label{app:exp-details}
The deep RL experiments with the catcher game are conducted as follows.
First, we use a uniform random policy to collect $50$k transitions in the environment. 
In this simple environment, the uniform random policy suffices to visit all states (the random agent is capable of accidentally catching the fruit, for example).
Next, each representation learning objective is maximized on this dataset. 
For all objectives, the images are pre-processed in the same manner (resized to $64$x$64$ pixels and normalized) and embedded with a convolutional network.
The convolutional encoder consists of five convolutional layers with ReLU activations and produces a latent vector with dimension $256$. 
We use the latent vector to estimate each mutual information objective, as described below.

\textbf{Inverse information}: We interpret the latent embeddings of the images $S_t$ and $S_{t+1}$ as the parameters of Gaussian distributions $p(Z | S_t)$ and $p(Z | S_{t+1})$. We obtain a single sample from each of these two distributions, concatenate them and pass them through a single linear layer to predict the action. The objective we maximize is the cross-entropy of the predicted actions with the true actions, as in Agrawal et al. 2016 and Shelhamer et al. 2016. To prevent recovering the trivial solution of preserving all the information in the image, we add an information bottleneck to the image embeddings. We tune the Lagrange multiplier on this bottleneck such that the action prediction loss remains the same value as when trained without the bottleneck. This approximates the objective $\min_{\phi} I(Z; S) s.t. I_{inv} = \max I_{inv}$. To use the learned encoder for RL, we embed the image from the current timestep and take the mean of the predicted distribution as the state for the RL agent. 

\textbf{State-only information}: We follow the Noise Contrastive Estimation (NCE) approach presented in CPC (Oord et al. 2018). Denoting $Z_t$ and $Z{t+1}$ as the latent embedding vectors from the convolutional encoders, we use a log-bilinear model as in CPC to compute the score: $f(Z_t, Z_{t+1}) = \exp (Z_t^T W  Z_{t+1})$ for the cross-entropy loss. We also experimented with an information bottleneck as described above, but found that it wasn’t needed to obtain insufficient representations. To use the learned encoder for RL, we embed the image from the current timestep and use this latent vector as the state for the RL agent.

\textbf{Forward information}: We follow the same NCE strategy as for state-only information, with the difference that we concatenate the action to $Z_t$ before computing the contrastive loss.

We then freeze the state encoder learned via MI-maximization and use the representation as the state input for RL.
The RL agent is trained using the Soft Actor-Critic algorithm~\cite{haarnoja2018soft}, modified for the discrete action distribution (the Q-function outputs Q-values for all actions rather than taking action as input, the policy outputs the action distribution directly rather than parameters, and we can directly compute the expectation in the critic loss rather than sampling). 
The policy and critic networks consist of two hidden linear layers of $200$ units each with ReLU activations. 
We run $5$ random seeds of each experiment and plot the mean as a solid line and one standard deviation as a shaded region.

Our implementation for the deep RL experiments is based on Garage~\citep{garage} and we use the default hyperparameters for the SAC algorithm found in that repository.
Garage is distributed under the MIT License. 
The environments we use are from the pygame library ~\citep{pygame} which is distributed under GNU LGPL version 2.1.
We used NVIDIA Titan X graphics cards to accelerate the training of the representation learning algorithms with image inputs.

\subsection{Alternative evaluation of the representation by predicting $Q^*$}
\label{app:expertq}
\begin{figure}[!t]
        \centering
          \includegraphics[width=6cm, trim=0.25cm 0 0 0, clip]{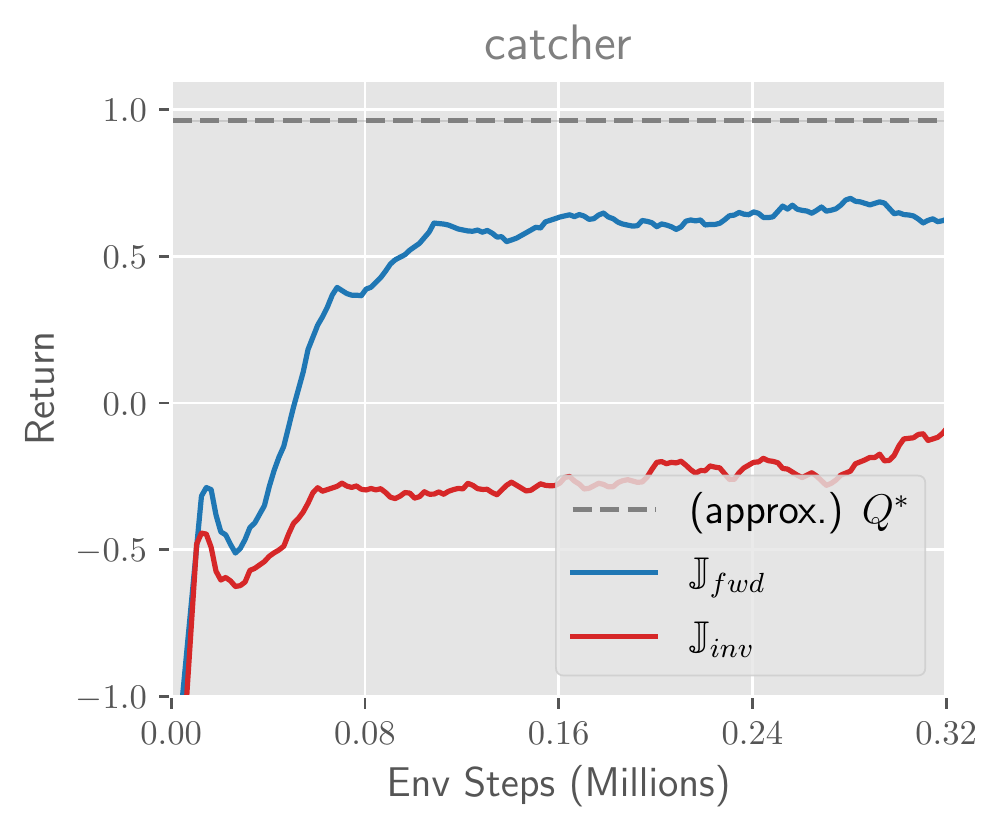}
          \includegraphics[width=6cm, trim=0.25cm 0 0 0, clip]{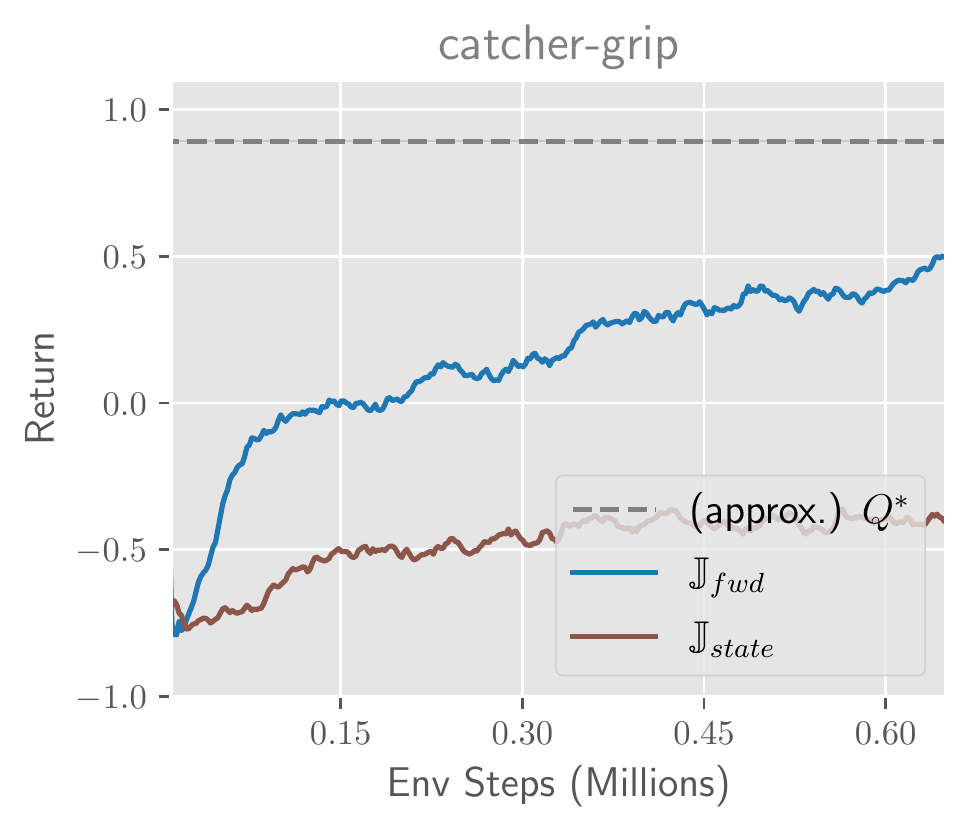}
          \caption{Policy performance obtained from $Q$-functions trained to predict $Q^*$ given state representations learned by each MI objective. Insufficient objectives $\inv$ and $\cpc$ respectively perform worse than sufficient objective $\fwd$.}
          \label{fig:predictQ}%
\end{figure}
In Section~\ref{sec:exps}, we evaluated the learned representations by running a temporal difference RL algorithm with the representation as the state input.
In this section, instead of using the bootstrap to learn the $Q$-function, we instead regress the $Q$-function to the optimal $Q^*$.
To do this, we first compute the (roughly) optimal $Q^*$ by running RL with ground truth game state as input and taking the learned $Q$ as $Q^*$.
Then, we instantiate a new RL agent and train it with the learned image representation as input, regressing the $Q$-function directly onto the values of $Q^*$.
We evaluate the policy derived from this new $Q$-function, and plot the results for both the \textit{catcher} and \textit{catcher-grip} environments in Figure~\ref{fig:predictQ}.
We find that similar to the result achieved using the bootstrap, the policy performs poorly when using representations learned by insufficient objectives ($\inv$ in \textit{catcher} and $\cpc$ in \textit{catcher-grip}).
Interestingly, we find that the error between the learned $Q$-values and the $Q^*$-values is roughly the same for sufficient and insufficient representations. 
We hypothesize that this discrepancy between $Q$-value error and policy performance is due to the fact that small differences in $Q$-values on a small set of states can result in significant behavior differences in the policy.

\subsection{Changing the data distribution used for representation learning}
\label{app:data-dist}
In Section~\ref{sec:exps}, we collect the dataset used to optimize the representation learning objectives with a uniform random policy. 
Here, we perform an ablation where we collect this dataset using the optimal policy, trained from ground truth state (agent and fruit positions).
We show the results for the \textit{catcher} environment in Figure~\ref{fig:catcher-opi}.
As predicted by our theoretical analysis, $\inv$ is still insufficient, and can fail to represent the fruit, leading to poor performance of the RL agent.

\subsection{Temporally correlated visual distractors}
\label{app:distractors}
In this section, we perform experiments with temporally correlated background distractors by animating the colored circles (see Figure~\ref{fig:igr_obs}) to bounce around the screen. 
We show results for the \textit{catcher} environment in Figure~\ref{fig:correlated}.
The results are similar to those in Figure~\ref{fig:results-igr} of the paper - the insufficient objective $\inv$ results in representations poorly suited to learning the task, while $\fwd$ yields a useful representation.
However, the $\fwd$ objective could suffer if the complexity of the background were further increased.
With limited model capacity, $\fwd$ might fixate on the moving background rather than the task-relevant fruit, because it is incentivised to predict all the state information that is predictable.
Regardless of the background complexity, $\inv$ is still insufficient in this MDP.

\begin{figure}[b]
    \begin{minipage}{.45\textwidth}
    \centering
    \includegraphics[width=6cm, trim=0.25cm 0 0 0.65cm, clip]{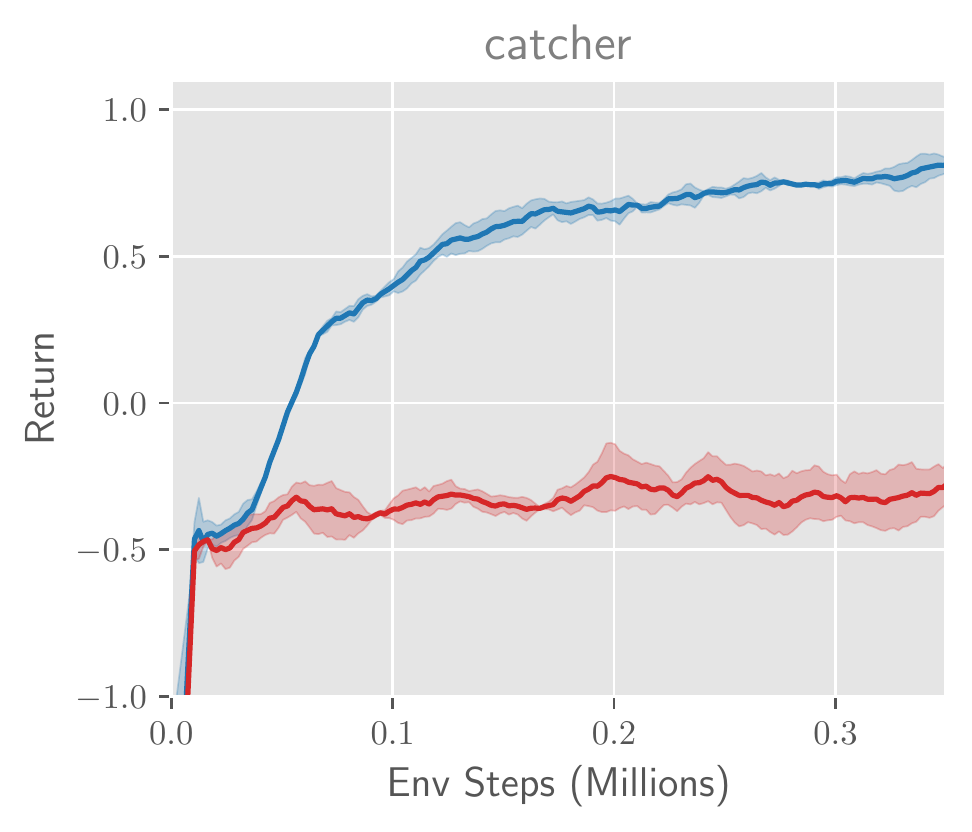} \\
    \includegraphics[height=.5cm, trim=0 0.3cm 0 0]{figs/legend}
    \caption{Performance of RL agents trained with state representations learned with data collected by the optimal policy, \textit{catcher} environment.}
    \label{fig:catcher-opi}
    \end{minipage}
    \hfill
    \begin{minipage}{.45\textwidth}
    \centering
    \includegraphics[width=6cm, trim=0.25cm 0 0 0.65cm, clip]{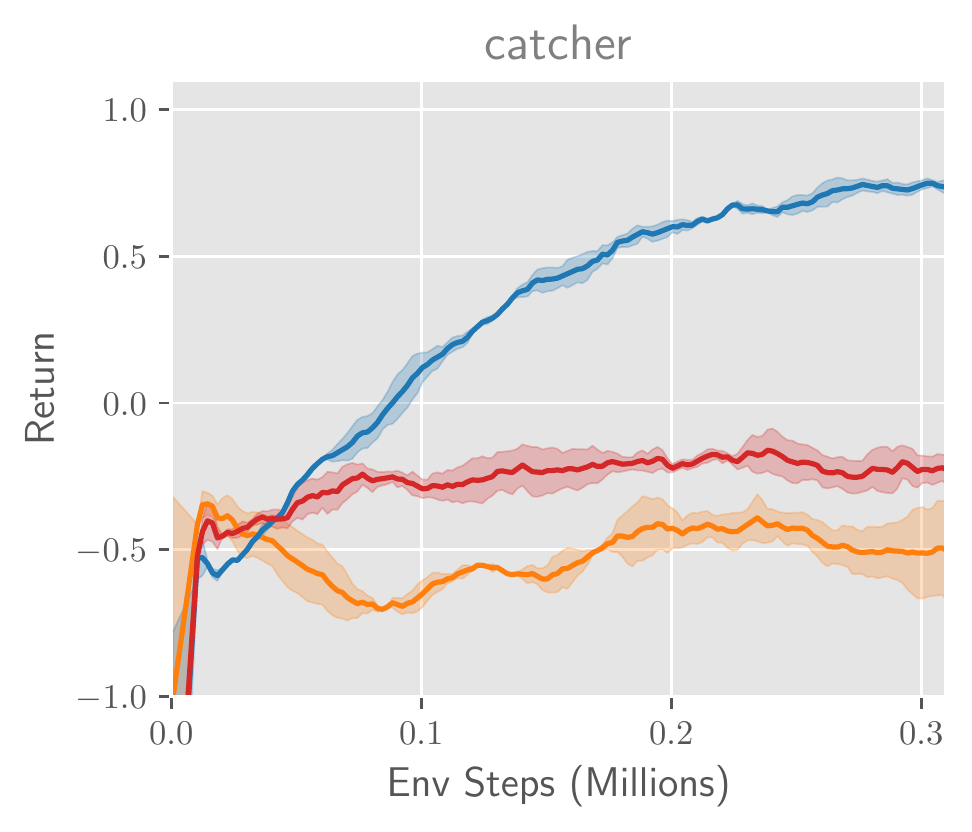} \\
    \includegraphics[height=.5cm, trim=0 0.3cm 0 0]{figs/legend}
    \caption{Performance of RL agents learned from state representations when observations contain visual distractors that are correlated across timesteps, \textit{catcher} environment.}
    \label{fig:correlated}
    \end{minipage}
\end{figure}

\end{document}